\documentclass[twoside,11pt]{article}
\usepackage{jair, theapa, rawfonts}

\usepackage[capitalize,nameinlink,noabbrev]{cleveref}
\usepackage{makeidx}  % allows for indexgeneration
\usepackage{graphicx}
\usepackage{epsfig}
\usepackage{epstopdf}
\usepackage{amsfonts}
\usepackage{amssymb}
\usepackage{amsmath}
\usepackage{amsthm}
\usepackage{amsfonts}
\usepackage{tabularx}
\usepackage{cases}
\usepackage{mathrsfs}
\usepackage{etoolbox}
\usepackage{url}
\usepackage{bm}
\usepackage{multirow}
\usepackage{booktabs}
\usepackage[noend]{algpseudocode}
\usepackage{algorithm}
\usepackage{enumerate}
\usepackage{float}
\usepackage[caption=false]{subfig}
\usepackage{color}
\usepackage[title]{appendix}
\usepackage{listings}
\newtheorem{theorem}{Theorem}[section]
\newtheorem{definition}{Definition}[section]

\newtheorem{proposition}{Proposition}[section]

\makeatletter

\newcommand{\Rmnum}[1]{\expandafter\@slowromancap\romannumeral #1@}
\makeatother

\jairheading{66}{2019}{473-502}{06/2018}{10/2019}
\ShortHeadings{MIRL for Certain General-Sum Stochastic Games}
{Lin, Adams \& Beling}
\firstpageno{473}

\begin{document}

\title{Multi-agent Inverse Reinforcement Learning for Certain General-Sum Stochastic Games}

\author{\name Xiaomin Lin \email xl5db@virginia.edu \\
        \addr Data Science, MassMutual Financial Group \\
        470 Atlantic Ave., Boston, MA 02210 USA \\
        \addr University of Virginia \\
       Charlottesville, VA 22904 USA \\
       \name Stephen C. Adams \email sca2c@virginia.edu \\
       \name Peter A. Beling \email pb3a@virginia.edu \\
       \addr University of Virginia \\
       Charlottesville, VA 22904 USA
       }
    %   \AND
    %   \name Philip Laird \email laird@ptolemy.arc.nasa.gov \\
    %   \addr NASA Ames Research Center,
    %   AI Research Branch, Mail Stop: 269-2,\\
    %   Moffett Field, CA  94035 USA}

% For research notes, remove the comment character in the line below.
% \researchnote

\maketitle

\begin{abstract}
%This paper attempts to address the problem of \emph{multi-agent inverse reinforcement learning} (MIRL) in a two-player \emph{general-sum} stochastic game framework. Three variants of MIRL, according to the strategy/equilibrium employed, are considered, where players: 1. fully cooperate to maximize the sum of players' game values, namely \emph{uCS-MIRL} for a cooperative game; 2. partially cooperate but follow a Correlated Equilibrium \emph{uCE-MIRL}, and; 3. partially cooperate but follow a Nash Equilibrium \emph{uNE-MIRL}. For uCS-MIRL, we propose a Bayesian approach, aiming to maximize the posterior of reward given the observed bi-policy. For the latter two problems, we first characterize the solutions ensuring the observed bi-policy constitutes a CE/NE and then propose some novel heuristics to pick a solution that not only minimize the total game value difference between the observed bi-policy and its \emph{local} uCS, but also maximize the scale of the solution. All these approaches solve the corresponding problems in a setting of efficiently solvable linear programming. We demonstrate our algorithms on two benchmark grid games.      
%\keywords{inverse reinforcement learning, multi-agent, game theory, general-sum, correlated equilibrium, Nash equilibrium, cooperative games, non-cooperative games}
This paper addresses a subset of problems of multi-agent inverse reinforcement learning (MIRL) in a two-player general-sum stochastic game framework. Specifically, five variants of MIRL are considered: \emph{uCS-MIRL}, \emph{advE-MIRL}, \emph{cooE-MIRL}, \emph{uCE-MIRL}, and \emph{uNE-MIRL}, each distinguished by its solution concept. Problem uCS-MIRL is a cooperative game in which the agents employ cooperative strategies that aim to maximize the total game value. In problem uCE-MIRL, agents are assumed to follow strategies that constitute a correlated equilibrium while maximizing total game value. Problem uNE-MIRL  is similar to uCE-MIRL in total game value maximization, but it is assumed that the agents are playing a Nash equilibrium. Problems advE-MIRL and cooE-MIRL assume agents are playing an adversarial equilibrium and a coordination equilibrium, respectively. We propose novel approaches to address these five problems under the assumption that the game observer either knows or is able to accurately estimate the policies and solution concepts for players. For uCS-MIRL, we first develop a characteristic set of solutions ensuring that the observed bi-policy is a uCS and then apply a Bayesian inverse learning method. For uCE-MIRL, we develop a linear programming problem subject to constraints that define necessary and sufficient conditions for the observed policies to be correlated equilibria. The objective is to choose a solution that not only minimizes the total game value difference between the observed bi-policy and a local uCS, but also maximizes the scale of the solution. We apply a similar treatment to the problem of uNE-MIRL. The remaining two problems can be solved efficiently by taking advantage of solution uniqueness and setting up a convex optimization problem. Results are validated on various benchmark grid-world games. 
\end{abstract}

\section{Introduction}\label{sec:intro}

This paper addresses the generalization of inverse reinforcement learning (IRL) to a multi-agent setting. IRL has been widely studied in recent years as part of a broad and growing interest in reinforcement learning (RL). The RL problem is commonly framed in terms of a single agent that aims to learn an optimal control policy for a Markov decision process (MDP) with reward function and state transition probabilities that are either known explicitly or that can be experienced though interaction with the environment. IRL, the inverse problem, has the objective of estimating the reward function given observations of a control policy followed by an agent \cite{Ng2000}. IRL has been applied to a number of problems, most related to the problems of learning from demonstrations and apprenticeship learning \cite{abbeel2004apprenticeship,Baker2009,Ziebart2008}. 

A key assumption in IRL is that a behavioral model for the agent is known. The most common behavior model is that the agent has acted optimally with respect to the MDP. Other behavioral assumptions can be adopted, however, such as a probabilistic selection of suboptimal behavior that attempts to model agents observed in the midst of their own learning process. IRL is inherently ill-defined, as more than one reward function can be consistent with an observed optimal policy (one way to see this is to note that any policy is optimal with respect to a reward function that is everywhere zero). The ill-defined nature of the problem typically is addressed by formulating an optimization objective that imposes additional structure or properties on the learned rewards. Two prominent examples of such problems are max-margin IRL \cite{abbeel2004apprenticeship}, in which the rewards are chosen so as to maximize the value function of the observed policy relative to all other policies, and Bayesian IRL \cite{ramachandran2007bayesian}, which aims to find maximum {\em a posteriori} estimate of rewards given priors and a likelihood model.  

RL and IRL can be extended to multi-agent settings. Multi-agent reinforcement learning (MRL) is defined in terms of a stochastic game, rather than an MDP as in the single-agent case (see, e.g., \citeR{Owen1968,Shapley1953}). In a stochastic game, the players are presented with a sequence of games with the joint actions taken in each game determining the transition probabilities to the next game. Hence, games play the role of states in an MDP. The rewards experienced by each player are determined by the payoff matrices of the games and the joint actions (or {\em bi-policy}) of the players. To define MRL as a computational problem, it is necessary to specify a solution concept, such as a Nash equilibrium or a correlated equilibrium that the players attempt to achieve. MRL algorithms have been developed for a number of equilibria with varying success in terms of theoretical and empirical performance \cite{Abdallah2008,Cigler2013,Conitzer2007,Greenwald2003,Hu1998,Littman1994,Littman2001,Sodomka2013}.  

The inverse learning problem for MRL, which we term multi-agent inverse reinforcement learning (MIRL), is to estimate the payoffs of a stochastic game given only observations of the actions taken by the players. Like its single-agent counterpart, the MIRL problem must be stated in terms of an assumed behavior model for the agents. Typically, the assumption is that the agents are playing equilibrium strategies. If one assumes that the agents reached the equilibrium by following an MRL algorithm, then it makes sense in MIRL to focus on developing solution methods for problems defined in terms of equilibria that have existing algorithms for the forward problem. 

MIRL can be viewed as a generalization of IRL in the sense that the later treats other agents in the system as part of the environment, ignoring the difference between decision-making agents and the passive environment. A financial trading example can be used to illustrate the difference. \citeA{Yang2015} use the reward functions inferred from IRL as a feature space for the purpose of classifying high-frequency trading algorithms in the stock market. This model is reasonable for a typical stock trading market. Usually there are many traders involved and their activities give rise to cancellation effects that make it reasonable for any one trader to model the collective actions of all the other traders as a stochastic system.  However, if the market is dominated by a small number of traders--such as is the case currently with crypto-currency trading--the market should not be modeled as a passive system. Rather, each dominant trader should take the other's possible strategies into account before making decisions. In such a case, a stochastic game framework (used in MIRL) would be more appropriate than a MDP framework (used in IRL).

MIRL has been studied far less than IRL, though there has been some relevant work on the topic in recent years. Natarajan et al. address multi-agent problems using an IRL model for multiple agents without dealing with interactions or interference among agents \cite{Natarajan2010}. \citeA{Waugh2011} contribute to the inverse equilibrium problem in the context of simultaneous one-stage games, rather than the sequential stochastic games that are the subject of MIRL. \citeA{Reddy2012} use the concept of a subgame perfect equilibrium (SPE) \cite{Maskin2001}, a refinement of the Nash equilibrium used in dynamic games, to address MIRL for general-sum stochastic games that have the property that each player's rewards do not depend on the actions of the others. \citeA{Hadfield-Menell2016} introduce a cooperative IRL problem, motivated from an autonomous system design problem, where the robot is required to align its value with those of the humans in its environment in such a way that its actions contribute to the maximization of values for the humans. Their problem is not modeled as a MIRL problem in a stochastic game context. \citeA{Lin2018} develop a Bayesian maximum {\em a posteriori} (MAP) estimation method. Their work is limited to two-person zero-sum stochastic games and is not applicable to arbitrary general-sum problems because the uniqueness property of the minimax equilibrium in zero-sum games does not carry over to solution concepts, such as the Nash equilibrium, that are important in general-sum games. In particular, the non-uniqueness of equilibria makes it unclear how to specify the likelihood function for a Bayesian inverse learning formulation. \citeA{Wang2018} propose an algorithm that takes sub-optimal demonstrations into account, rather than the optimal strategies assumed by \citeA{Lin2018}, and finds the reward function to minimize the margin between experts' performance and Nash equilibria-directed results.

This paper studies five two-person general-sum MIRL problems, \emph{uCS-MIRL}, \emph{advE-MIRL}, \emph{cooE-MIRL}, \emph{uCE-MIRL}, and \emph{uNE-MIRL}, each distinguished by a solution concept and corresponding class of equilibrium policies that the observed agents should be assumed to be playing. The first problem, uCS-MIRL, is a \emph{cooperative game} in which the agents employ cooperative strategies (CSs) that aim to maximize the sum of their value functions, or the \emph{total game value}. The second and third problems consider circumstances in which two special and unique Nash equilibria (NE) are employed: advE is in general a \emph{win-or-lose} equilibrium, but not necessarily for a zero-sum game; cooE is such an equilibrium that players maximize their own payoffs by coordinating with others. In the fourth problem, uCE-MIRL, the agents are assumed to follow strategies that constitute a \emph{utilitarian correlated equilibrium} (uCE), which achieves the maximum total game value among all CEs. In the last problem, uNE-MIRL, players are assumed to follow strategies that constitute a NE that maximizes total game value.

The equilibria that we study from an inverse perspective arise in a variety of application contexts. For example, uCS-MIRL, uCE-MIRL, and uNE-MIRL embed solution concepts that correspond to agents trying to achieve a socially efficient outcome (with or without certain constraints) that maximizes the sum of their value functions and is a Pareto optimum, meaning that it is not possible to make one player better off without making the other player worse off \cite{Barr2012}. These equilibria are of particular interest in welfare economics. For example, uCS applies to the situation where all players of a basketball team cooperate to maximize the total points rather than boosting any single player's performance. uCE is motivated from the social governance problem where policy makers are required to design rules to achieve Pareto optimality in social welfare without harming individual interests. Note that uCE has been studied from an MRL perspective \cite{Greenwald2003}. Despite the fact the advE and cooE-MIRL equilibria are not guaranteed to exist in every game, they have been studied in MRL \cite{Littman2001} and have potential applications. Consider an example in which two power suppliers compete with each other in a local market. Though to each supplier, this is obviously a win-or-lose game, it is usually not likely to evolve to a dominate-or-exit situation. Hence it might be more reasonable to assume the suppliers are playing an advE equilibrium rather than a minimax solution to a zero-sum game. As for cooE, the classic Stag Hunt game (see, e.g. \citeR{Skyrms2004}) is representative of a broad range of social cooperation games. In Stag Hunt there are two hunters, each can chose to hunt hare or stag, with symmetric payoffs. If they both hunt stag(hare), they both will get a payoff of 2(1); and if their targets are different, the one who hunts stag will fail to get anything and the other will get a payoff of 1. In this game, (stag, stag) is a cooE. 

In addition to their importance in applications, the five solution concepts that we consider each has been studied from the MRL (i.e., forward) perspective and computational algorithms exist that players of these games could, if they chose, follow to reach an equilibrium solution. Equilibrium uCS is actually an extension of RL to a multi-agent context in which the RL optimality concept is still valid. Equilibria cooE, advE and uCE have been well studied as a forward problem and corresponding Q-learning algorithms that guarantee convergence to these equilibria have been developed  \cite{Greenwald2003,Hu1998,Littman2001}. Equilibrium uNE has similar properties as uCE and is also computationally achievable. Hence, each of the five MIRL problems (\emph{uCS-MIRL}, \emph{advE-MIRL}, \emph{cooE-MIRL}, \emph{uCE-MIRL}, \emph{uNE-MIRL}) that we study is based on an equilibrium concept that might be followed by a pair of decision-makers playing a stochastic game. Given that these problems cover cooperative, semi-cooperative and fully competitive games, we argue that they represent a reasonable starting point for the study of MIRL in general-sum, stochastic games.

The principal contributions of this paper lie in framing and deriving solution methods for five MIRL problems.  For uCS/advE/cooE-MIRL, the key step is the development of a set of linear constraints on the reward function that are necessary and sufficient for the observed bi-policy to be a unique equilibrium of the assumed type. We then show that, by using these conditions in a Bayesian MAP-estimation formulation similar to that developed by \citeA{Lin2018} for zero-sum games, uCS/advE/cooE-MIRL can be solved as a convex optimization problem. An alternate approach, for which we do not provide details, would be to use the same necessary and sufficient conditions as the basis for a max-margin formulation similar to those seen in IRL \cite{Natarajan2010,Ng2000,Reddy2012}. 

Linear conditions specifying a unique equilibrium are not known for uCE/uNE.  This circumstance forces us to abandon the MAP estimation formulation used in uCS/advE/cooE-MIRL and adopt an equilibrium selection approach. For uCE-MIRL, we derive linear conditions on reward that are necessary and sufficient for the observed bi-policy to be a CE, a class of equilibria that includes the unique uCE. These conditions are used in a linear program with a novel objective function that incorporates both the scale of the solution and a metric on the total game value difference between the observed bi-policy and a notion of a local equilibrium. We apply a similar treatment to the uNE-MIRL problem. From a high-level perspective, the idea behind our uCE/uNE-MIRL algorithm completely departs from previous ones from which many IRL/MIRL algorithms stem, such as margin-maximization, posterior maximization and entropy maximization. The ideas developed in this paper--while limited to five problems and to two-player settings--could be extended to other uniquely existing solution concepts and to $n$-player stochastic games.

The remainder of this paper is structured as follows.  \cref{sec:preliminary_general_sum} introduces notation, terminology and definitions that will be used throughout this paper, as well as some basic game theory equilibrium concepts. \cref{sec:current_mirl} summarizes several conventional MIRL algorithms. \cref{sec:mirl_general_sum} provides the main technical work, developing different approaches for different problems to learn rewards. \cref{sec:experiments_general_sum} and \cref{sec:soccer} demonstrate our approaches through several benchmark experiments that include comparison with existing MIRL algorithms. Taking uNE-MIRL algorithm as an example, \cref{sec:incomplete_policy} explores the robustness of our algorithms given incomplete observations. \cref{sec:general_sum_conclusion} offers concluding remarks and a discussion of future work.

\section{Preliminaries} \label{sec:preliminary_general_sum}

This section serves to introduce concepts and notation for MRL that will be used throughout the paper. It also introduces relevant concepts and formalism for two-player general-sum games and the equilibria of interest in later sections.
 
\subsection{Stochastic Game}
A two-player general-sum discounted stochastic game is a tuple $\left \{ \mathcal{S}, \mathcal{A}_i, {R}_i, P, \gamma  \right \}$, where $\mathcal{S}$ is the common state space for all players, $\mathcal{A}_i$ and $R_i$ are the action space and reward for player $i$, respectively. $P$ is the probabilistic function controlling state transitions, conditioned on the past state and joint actions.  The reward discount factor is $\gamma \in \left [ 0, 1 \right )$. In this paper, we assume that both players share the same action space. The state and action spaces are both finite, i.e. $\left | \mathcal{S} \right |=N$ and $\left | \mathcal{A}_i \right |=M$. A stochastic game is a sequence of single-stage games, or \emph{subgames}, induced in every state $s\in \mathcal{S}$, such that both players need to determine an \emph{individual strategy} $\pi_i\left ( s \right )$ (in a non-cooperative game) or coordinate a \emph{bi-strategy} $\pi\left ( s \right )$ (in a cooperative game) that guides their actions in every subgame. The collection of all bi-strategies is a \emph{bi-policy} $\pi = \left ( \pi_1, \pi_2 \right )=\left ( \pi_1\left ( s \right ), \pi_2\left ( s \right ) \right ), \forall s \in S$. Note that an individual strategy can be a mixed strategy, which is a probability distribution over all available actions. We define a \emph{pure} bi-strategy $a \in \mathcal{A}=\mathcal{A}_1 \times \mathcal{A}_2$ as a bi-strategy where both players select deterministic actions.  Each player's reward values are assumed dependent on state and possibly, bi-strategies, but are independent of each other. 

\subsection{Multi-agent Problems in RL/IRL Context} \label{subsec:multi_RL_IRL}

%\textcolor{red}{In multi-agent scenario, we do not think the term ``MRL/MIRL" as unambiguous as ``RL/IRL". The reason is that RL/IRL has been strictly defined, in particular, no disagreement with what the concept of optimality, or solution concept is and what the objective is. However, for MRL/MIRL, we have so far seen different definitions of ``optimality" and different objectives (one single common value maximization, individual value maximization, etc.). While most people agree that the solution concept of a MRL/MIRL problem should be an equilibrium that is stable, there is no consensus on which equilibrium to choose. For example, in \cite{Littman1994,Hu2003,Natarajan2010}, authors select a general Nash equilibrium; in \cite{Greenwald2003}, five variants of correlated equilibrium are raised. Therefore, in our mind, MRL/MIRL is a class of problems as against to RL/IRL, which is a single problem. We claim that it is more appropriate to define each MRL/MIRL problem on a solution basis rather than proposing a single definition to cover all problems. One potential benefit is that it may be easier to find a solution by narrowing down the scope.}
 
We now introduce some basic terms, notations and fundamental equations that will be used throughout this paper. The derivations, developed by \citeA{Lin2018} are omitted here. $\tilde{r}_i^{\pi}\left ( s \right )$ denotes the \emph{expected reward value} received by agent $i$ at state $s$ under bi-policy $\pi$, specifically, 
\begin{equation}\label{R_average}
\begin{aligned}
\tilde{r}_i^{\pi}\left ( s \right ) &=  \sum_{a}\pi_1\left ( a_1 |s\right )\pi_2\left ( a_2|s \right )R_i\left ( s, a \right ) \\
&=  \left [ \pi_1\left ( s \right ) \right ]^TR_i\left ( s \right )\pi_2\left ( s \right ), \forall s \in \mathcal{S},
\end{aligned}
\end{equation}
where $a_i$ is one of player $i$'s available actions, and $a$ is a pure bi-strategy. $\pi_i\left ( s \right )$ is a $M \times 1$ vector denoting the probability distribution over actions in state $s$, each entry of which is the probability of taking action $a_i$ at state $s$, denoting $\pi_i\left (a_i| s \right )$. $R_i\left ( s \right )$ is a $M \times M$ matrix, each entry of which denotes a pure bi-strategy dependent reward value. Structuring all $R_i\left ( s, a \right )$ into a column vector as $r_i$, we can simplify and represent \eqref{R_average} in matrix notation as 
\begin{equation}\label{r_ave}
\tilde{r}_i^{\pi}=B_{\pi}r_i.
\end{equation}
The linear transformation operator $B_{\pi}$ is a $N \times NM^2$ matrix constructed from $\pi$, whose $k$th row is:
\begin{equation*}
\left [ \Phi^{\pi}_{1,1}\left ( k \right ), \Phi^{\pi}_{1,2}\left ( k \right ),\cdots ,\Phi^{\pi}_{M,M}\left ( k \right )   \right ],
\end{equation*}
where 
\begin{equation*}
\Phi^{\pi}_{i,j}\left ( k \right ) = \left [ \underbrace{0,\cdots ,0}_{k-1},\phi^{\pi}_{i,j}\left ( k \right ), \underbrace{0,\cdots ,0}_{N-k} \right ],
\end{equation*}
and 
\begin{equation*}
\phi^{\pi}_{i,j}\left ( k \right )=\pi^1\left ( i|k \right )\pi^2\left ( j|k \right ).
\end{equation*}
Player $i$'s \emph{value function}, starting at state $s$ and under $\pi$, is defined as  
\begin{equation}\label{V_orginial}
V_i^{\pi}\left ( s \right ) = \sum_{t=0}^{\infty }\gamma^tE\left ( \tilde{r}_i^{\pi}\left ( s_t \right ) |s_0=s \right ),
\end{equation}
and its \emph{Q}-function, upon $s$ and $a$, is 
\begin{equation}\label{Q_element}
Q_i^{\pi}\left ( s,a \right )=r_i\left ( s, a \right )+\gamma P_{s,a}V_i^{\pi},
\end{equation}
where $P_{s,a}$ is a row vector, each entry of which represents the transition probability from $s$ to all possible state given action $a$. Furthermore,
\begin{equation}\label{Q_equilibria}
V_i^{\pi}\left ( s \right ) \in \mbox{solution concept}_i  \left ( Q_1^{\pi}\left ( s  \right ), Q_2^{\pi}\left ( s  \right )  \right ), \forall s \in \mathcal{S}.
\end{equation}
Presumably, players agree on a solution concept and choose individual strategies accordingly. In order to pose an inverse learning problem, the solution concept must be known to the game observers estimating the reward function because it cannot be inferred or observed when provided the actions of the players.
\par
Let $G_{\pi}$ denote a $N \times N$ transition matrix under bi-policy $\pi$, with elements
%\begin{equation*}
%G_{{\pi}}\left ( i,j \right )=\sum_{a^1, a^2}\pi^1\left ( i, a^1 \right )\pi^2\left ( i, a^2 \right )p\left ( j|i, a^1, a^2 \right )
%\end{equation*}
\begin{equation}\label{G_def}
g_{\pi}\left ( s' | s \right )=\sum_{a}\pi_1\left ( a_1|s \right )\pi_2\left ( a_2|s \right )p\left ( s'|s, a \right ).
\end{equation}
%\end{definition}
%Note that $G_{{\pi}}$ is a transition matrix because
%\begin{equation*}
%\sum_{k=1}^{N}G_{{\pi}}\left ( s,s' \right )=\sum_{k=1}^{N}p\left ( s'|s \right )=1
%\end{equation*}
%\par
%\eqref{V_orginial} can be expended as
Then  
% \begin{equation}\label{V_expend}
% %\begin{aligned}
% V_i^{\pi}\left ( s \right )=\tilde{r}_i^{\pi}\left ( s \right ) + \gamma\sum_{s'}g_{\pi}\left ( s'|s \right )V_i^{\pi}\left ( s' \right ).
% %\end{aligned}
% \end{equation}
% In addition, $V_i^{\pi}\left ( s \right )$ can also be expressed in terms of the $Q$-function as
% \begin{equation}\label{shapley_formula}
% V_i^{\pi}\left ( s \right ) = \left [ \pi_1\left ( s \right ) \right ]^T Q_i^{\pi}\left ( s \right )\pi_2\left ( s \right ),
% \end{equation}
% where $Q_i^{\pi}\left ( s \right )$ is a $M \times M$ matrix. 
% We can rewrite \eqref{V_expend} in matrix notation as
% \begin{equation}\label{V_comp}
% V_i^{\pi}=\tilde{r}_i^{\pi}+\gamma G_{\pi}V_i^{\pi}.
% \end{equation}
% Thus
\begin{equation}\label{V_comp2}
V_i^{\pi}=\left ( I - \gamma G_{\pi} \right )^{-1}B_{\pi}r_i,
\end{equation}

\begin{equation}\label{Q_complicated}
\vec{Q}_i^{{\pi}} = \left ( I + \gamma P\left ( I - \gamma G_{{\pi}} \right )^{-1}B_{{\pi}} \right ){r}_i.
\end{equation}
$V_i^{\pi}$ can be rewritten more compactly as

% Restructuring $Q_i^{\pi}\left ( s,a \right )$ into a column vector, denoting $\vec{Q}_i^{\pi}$, we can rewrite equation \eqref{Q_element} in matrix notation, over all states and joint actions, as
% \begin{equation}\label{Q_all}
% \vec{Q}_i^{\pi}=r_i+\gamma P V_i^{\pi}, 
% \end{equation}
% where $P$ is a $NM^2 \times N$ matrix with $p\left ( s'|s, a \right )$ as its elements. Combining \eqref{Q_all} and \eqref{V_comp2} leads to 
% %Since $G_{{\pi}}$ is a \emph{transition matrix}, $G_{{\pi}}$ has all eigenvalues in the unit circle in the complex plane. Additionally, $0 < \gamma < 1$ implies that the matrix $V_{{\pi}}$ has all eigenvalues in the interior of the unit circle. This means that $\left ( I-\gamma G_{{\pi}} \right )$ has no zero eigenvalues, and is thus not singular.
\begin{equation}\label{V_Q_B}
V_i^{{\pi}}=B_{{\pi}}\vec{Q}_i^{{\pi}}.
\end{equation}
Lastly, we define the \emph{total game value} of a two-player stochastic game starting at state $s$, under a bi-policy $\pi$, $V^{\pi}\left ( s \right ) $, as the sum of the value functions of both players, i.e., $V^{\pi}\left ( s \right )=V_1^{\pi}\left ( s \right )+V_2^{\pi}\left ( s \right )$.

\par

\subsection{Non-cooperative Equilibrium} \label{subsec:non_cooperative_equlibrium}
In non-cooperative game theory, \emph{Nash equilibrium} (NE) and \emph{correlated equilibrium} (CE) are two of the most important solution concepts. A NE is an equilibrium where no player will benefit from unilaterally deviating from their current strategy given the other players' strategies remain unchanged \cite{Nash1950,Nash1951}. In a two-player single-stage game (state $s$), $\pi\left ( s \right )$ is a NE if and only if,
\begin{equation} \label{r_ne}
R_i\left (s, \pi\left ( s \right )\right )\geq R_i\left ( s, a_i, \pi_{-i}\left ( s \right ) \right ), a_i \in \mathcal{A}_i,
\end{equation}
where $\pi_{-i}\left ( s \right )$ denotes the other player's strategy at state $s$.
Unlike NE, in which all agents act \emph{independently} on a \emph{selfish} and \emph{conservative} basis, a CE is a probability distribution over the joint space of actions, in which all agents optimize their payoff with respect to one another's probabilities, conditioned on their own probabilities \cite{Greenwald2003}. Let $\Delta \left ( \mathcal{A} \right )$ denote the set of probability distributions over $\mathcal{A}$, and $X$ be a random variable taking values in $\mathcal{A}=\prod_{i \in \mathcal{I}}\mathcal{A}_i$ distributed according to $\pi \in \Delta \left ( \mathcal{A} \right )$. Then $\pi$ is a correlated equilibrium if and only if \cite{Ozdaglar2010} 
\begin{equation*}
\Sigma_{a_{-i} \in A_{-i}} P\left ( X=a|X_i=a_i \right )\left [ R_i\left ( s, a_i,a_{-i} \right )-R_i\left (s, \check{a}_i,a_{-i} \right ) \right ] \geq 0,
\end{equation*}
for all $a_i \in \mathcal{A}_i$ such that $P\left ( X_i=a_i \right ) > 0$ and all $\check{a}_i \in \mathcal{A}_i \setminus a_i$.

It has been proved that every finite game has at least one NE \cite{Nash1951,Owen1968}, as well as at least one CE \cite{Hart1989}. In fact, CE is a superset of NE \cite{Aumann1974}, and hence for any general sum game, the number of CEs is larger than or equal to the number of NEs. In regard to equilibrium search, finding a NE or determining the number of them is a NP-hard problem \cite{Daskalakis2009}, while finding CEs can be done in polynomial time via linear programming \cite{Papadimitriou2008}. Nevertheless, the non-uniqueness property causes the non-convergence issue and is a bottleneck for MRL/MIRL research based on these two equilibria.

\subsection{Cooperative Strategy}
\label{subsec:cs_general_sum}
Both NE and CE  are equilibria of competitive games. In a cooperative game, by contrast, an agreement on a joint strategy for all players can be called a \emph{cooperative strategy} (CS). A \emph{characteristic function} $v$ defines the type of cooperation between players \cite{Ferguson2008}, and for a two-player single-stage game (state $s$), can be defined as 
\begin{equation} \label{eq:cs_val}
v\left ( s, a \right )=\mbox{Val}\left (R_1\left ( s, a \right ), R_2\left ( s, a \right )\right ), a \in \mathcal{A} = \mathcal{A}_1 \times \mathcal{A}_2.
\end{equation}
$\mbox{Val}\left ( \cdot \right )$ is self-defined, based on the type of cooperation. 
\section{Conventional MIRL Approaches} \label{sec:current_mirl}
Before introducing our approaches, we review several existing approaches to MIRL and related problems. The first of these is a \emph{decentralized} MIRL (d-MIRL) approach developed by \citeA{Reddy2012}. This approach is decentralized in the sense that it infers agents' rewards one by one, rather than all at once. All agents are assumed to follow a Nash equilibrium at every single game. The key idea is to find a reward that maximize the difference between the $Q$ value of the observed policy and those of pure strategies, which is analogous to the classical approach to single-agent IRL proposed by \citeA{Ng2000}. Though in the original version of their approach reward is assumed dependent only on the state, we can extend it to treat action dependency as well. Using our notation and player 1 as an example, the d-MIRL approach to a two-person general-sum MIRL problem is to solve the following linear program: 
\begin{equation*} 
%\begin{aligned}
%\textrm{maximize:}  \quad
%& \sum_{s=1}^{N}\mbox{min }_{a_1} \left ( \tilde{r}_1^{\pi}\left ( s \right )-\tilde{r}_1^{\pi|a_1}\left ( s \right ) \right ) \\
%& +\gamma \left ( G_{\pi}\left ( s \right )- G_{\pi|a_1}\left ( s \right ) \right )\left ( I-\gamma G_{\pi}\right )^{-1}B_{\pi}r_1   \\
%& -\lambda \left \| r_1 \right \|_1 \\
%\textrm{subject to:} \quad
%&\left ( B_{{\pi}|a_1}-B_{{\pi}} \right )D_{{\pi}}{r}_1\leq {0},\\
%\end{aligned}
\begin{aligned}
\textrm{maximize:}  \quad
& \sum_{s=1}^{N}\mbox{min }_{a_1} \left ( B_{{\pi}} - B_{{\pi}|a_1} \right )D_{{\pi}}{r}_1 -\lambda \left \| r_1 \right \|_1 \\
\textrm{subject to:} \quad
&\left ( B_{{\pi}|a_1}-B_{{\pi}} \right )D_{{\pi}}{r}_1\leq {0},\\
\end{aligned}
\end{equation*}
where $D_{\pi}$ is defined in \eqref{eq:uCS} in the following section, $\lambda$ is an adjustable penalty coefficient for having too many non-zero values in the reward vector. 

The key idea of the second approach is to model a two-person general-sum MIRL as an IRL problem. This approach requires us to select one player (e.g. player 1) and treat the other player as part of the passive environment. A Bayesian IRL (BIRL) algorithm developed by \citeA{Lin2018} extends \citeS{Qiao2011} work by considering action-dependent reward cases in addition to state-dependent reward cases. Note that the reward of player 1 to be recovered is $R_1\left ( s,a_1 \right )$ instead of $R_1\left ( s,a_1, a_2 \right )$, as player 2 is not considered adaptive. That is to say, $R_1\left ( s,a_1, a_2 \right )=R_1\left ( s,a_1 \right )$ for all $a_2 \in \mathcal{A}_2$. Using our notation, the approach to recover player 1's reward is: 
\begin{equation}\label{IRL_complex_program}
\begin{aligned}
\textrm{minimize:}  \quad
& \frac{1}{2}\left ( {r_1-{\mu_{r_1}}} \right )^T \Sigma^{-1}_{r_1} \left ( {r_1-{\mu_{r_1}}} \right )   \\
\textrm{subject to:} \quad
& \left ( F^{\pi_1}_{a_1} - C_{a_1}  \right ) r_1 \geqslant 0,
%&{r}_{\min} \leq {r} \leq {r}_{\max}
\end{aligned}
\end{equation}
for all $a_1 \in\mathcal{A}_1$, where
\begin{equation*}
F^{\pi_1}_{a_1} = \left [ \gamma\left ( G_{\pi} - G_{\pi_1|a_1} \right )\left ( I - \gamma G_{\pi}\right )^{-1} + I  \right ]C_{\pi_1},
\end{equation*}
and $C_{\pi_1}$ is a $N \times NM$ sparse matrix constructed from $\pi_1$, whose $i$th row is,
\begin{equation*}
\left [ \underbrace{0, \cdots, \pi^1\left ( i, 1 \right ), \cdots ,0}_{N},\underbrace{\cdots}_{\left ( M-2 \right )N} , \underbrace{0, \cdots,\pi^1\left ( i, M \right ), \cdots ,0}_{N}\right ],
\end{equation*}
and $C_{a_1}$ is conceptually similar to $C_{\pi_1}$, except for being constructed from a pure strategy $a_1$ for all states. 

Strictly speaking, the BIRL approach is not a dedicated algorithm for MIRL problems but rather a way of shoehorning the multi-agent problem into a single-agent IRL setting. BIRL will provide a useful point of comparison to quantify the benefits of explicitly modeling the decisions of all players. 

The third approach is not applicable to a general-sum MIRL problem but a restricted family: zero-sum games. It is
\begin{equation}\label{eq:zerosum_program}
\begin{aligned}
\textrm{minimize:}  \quad
& \frac{1}{2}\left ( {r-{\mu_r}} \right )^T \Sigma^{-1}_r \left ( {r-{\mu_r}} \right )   \\
\textrm{subject to:} \quad
&\left ( B_{\pi|a_1}-B_{{\pi}} \right )D_{{\pi}}{r}\leq {0}\\
\quad
&\left ( B_{\pi|a_2}-B_{{\pi}} \right )D_{{\pi}}{r}\geq {0},\\
%\quad
%&{r}_{\min} \leq {r} \leq {r}_{\max} 
\end{aligned}
\end{equation}
for all $a_1\in\mathcal{A}_1$ and $a_2\in\mathcal{A}_2$. \citeA{Lin2018} provide more details.

These three approaches will be revisited as benchmarks in later sections. 

\section{MIRL Model Development} \label{sec:mirl_general_sum}
This section proposes five two-player general-sum MIRL problems and corresponding approaches to them. We first informally define an MIRL problem. Given a bipolicy $\pi$ being played in a two-player, general-sum game with states, actions, dynamics, and discount $\left \{ \mathcal{S}, \mathcal{A}_i, P, \gamma\right \}$, the MIRL problem is to find rewards $r_1$, $r_2$ that best explain the observed policy. Though we will not do so in this paper, MIRL may defined in terms of an set of observed state-action values  $\mathcal{O}$ rather a bipolicy.  
\par
The MRL literature suggests that an agreement over a specific solution concept may be needed to solve a MRL problem. Similarly, in our approaches to MIRL, one basic assumption is required: both players agree on a specific strategy/equilibrium to play and this information is available to the coordinator in posing the MIRL problem. We limit attention to the following five solution concepts: 
\begin{enumerate}
\item \textbf{\emph{utilitarian} Cooperative Strategy (uCS)}. In \eqref{eq:cs_val}, consider $\mbox{Val}\left ( \cdot \right )=\sum \left ( \cdot \right )$. A single-stage game in state $s$ and taking action $a$ is a \emph{utilitarian} cooperative strategy (uCS) if and only if 
\begin{equation}
\sum_{i}R_i\left ( s, a \right ) \geq \sum_{i}R_i\left ( s, a' \right ), a' \in \mathcal{A}= \mathcal{A}_1 \times \mathcal{A}_2 \setminus a.
\end{equation}
\item \textbf{Adversarial Equilibrium (advE)} An advE is a type of NE. It has another feature that no player is hurt by any change of others \cite{Hu1998,Littman2001,piotrowski2010reinforced}. That is to say, in a two-player single-stage game (state $s$), $\pi\left ( s \right )$ is an advE if and only if, in addition to \eqref{r_ne},
\begin{equation} \label{r_adve}
R_i\left (s, \pi_i\left ( s \right ), \pi_{-i}\left ( s \right ) \right )\leq R_i\left ( s, \pi_i\left ( s \right ), a_{-i} \right ), a_{-i} \in \mathcal{A}_{-i},
\end{equation}
\item \textbf{Coordination Equilibrium (cooE)}. A cooE is also a type of NE. It has another feature that all players' maximum expected payoffs are achieved \cite{Hu1998,Littman2001,piotrowski2010reinforced}. Mathematically, in a two-player single-stage game (state $s$), $\pi\left ( s \right )$ is a cooE if and only if, in addition to \eqref{r_ne},
\begin{equation} \label{r_cooe}
R_i\left (s, \pi\left ( s \right ) \right )\geq R_i\left ( s, a \right ), a \in \mathcal{A}= \mathcal{A}_1 \times \mathcal{A}_2.
\end{equation}
\item \textbf{\emph{utilitarian} Correlated Equilibrium (uCE)}. We borrow the concept of \emph{utilitarian} correlated equilibrium (uCE) defined by \citeA{Greenwald2003} and state that in a two-player single-stage game (state $s$), $\pi\left ( s \right )$ is a uCE if and only if, 
\begin{equation}
\Sigma_i R_i\left (s, \pi\left ( s \right )\right ) \geq \Sigma_i R_i\left (s, \check{\pi}\left ( s \right )\right ), \check{\pi}\left ( s \right ) \in \pi_{\text{CE}} \setminus \pi\left ( s \right ).
\end{equation} 
\item \textbf{\emph{utilitarian} Nash Equilibrium (uNE)}. Similar to uCE, in a two-player single-stage game (state $s$), a NE $\pi\left ( s \right )$ is a \emph{utilitarian} Nash Equilibrium (uNE) if and only if 
\begin{equation}
\Sigma_i R_i\left (s, \pi\left ( s \right )\right ) \geq \Sigma_i R_i\left (s, \check{\pi}\left ( s \right )\right ), \check{\pi}\left ( s \right ) \in \pi_{\text{NE}} \setminus \pi\left ( s \right ).
\end{equation}
\end{enumerate}
Among the above five equilibria, it is easy to show that uCS, uCE and uNE always exist and are unique in any games (uCS for cooperative while uCE and uNE for noncooperative). Both advE and cooE have been shown to be unique in a noncooperative game, though neither of them is guaranteed to exist \cite{Hu1998,Littman2001} in any games. 

The distinctions between cooE and uCS are worth noting. Intuitively, cooE is a noncooperative game equilibrium, which means that agents are essentially selfish. They are \emph{forced} to cooperate in order to maximize their individual benefits. When following a uCS, by contrast, agents cooperate with each other \emph{actively} and may even sacrifice their own benefits to achieve a better overall outcome. \cref{sec:experiments_general_sum} will help illustrate the differences.

\subsection{Extension to Stochastic Games} 
\citeA{Filar1996} show how the $Q$ function links a stochastic game to a single stage game. $Q$ functions at one particular state with different bi-strategy are treated as payoffs for that particular single stage game (note the terms ``game" and ``state" can be used interchangeably), and the stochastic game is said to be \emph{in an equilibrium} if and only if all single games (over all states) are in equilibrium. We now extend our definitions of the five strategies/equilibria from a single game to a two-player stochastic game, as follows,
\begin{definition}
A bi-policy $\pi$ is a uCS/advE/cooE/uNE/uCE of a two-player stochastic game $\mathcal{G}$ if only if $\pi\left ( s \right )$ is a uCS/advE/CooE/uNE/uCE of its sub-game $\mathcal{G}\left ( s \right )$, for all $s \in \mathcal{S}$.
\end{definition}
\par
Correspondingly, we define that a uCS/advE/cooE/uNE/uCE-MIRL problem is an MIRL problem in which the players are assumed to employ a uCS/advE/CooE/uNE/uCE. 

\subsection{uCS-MIRL}
A main result characterizing the set of solutions to a two-player uCS-MIRL problem is the following:
\begin{theorem}
Given a two-player stochastic game $\left \{ \mathcal{S}, \mathcal{A}_i, r_i, P, \gamma  \right \}$, an observed bi-policy $\pi$ is a uCS if and only if 
\begin{equation}\label{eq:uCS} 
\left (B_{\pi}-B_{a}  \right )D_{\pi}\left ( r_1+r_2 \right ) \geq 0,a \in \mathcal{A}=\mathcal{A}_1\times \mathcal{A}_2 
\end{equation}
where $D_{{\pi}}=I + \gamma P\left ( I - \gamma G_{{\pi}} \right )^{-1}B_{{\pi}}$. $B_{a}$ is obtained from such a bi-policy that players employ the bi-strategy $a$ in all states.
\end{theorem}
\begin{proof}
According to the definition of uCS, $\pi$ is a uCS if and only if, for any state $s$ and pure bi-strategy $a \in \mathcal{A}=\mathcal{A}_1\times \mathcal{A}_2$, we have
\begin{equation}
\begin{aligned}
&\pi\left ( s \right )\in \arg \max_{a \in \mathcal{A}}\sum_{i}Q^{\pi}_i\left ( s, a \right ) \\
\Leftrightarrow &\sum_{i}Q^{\pi}_i\left ( s, \pi\left ( s \right ) \right )\geq \sum_{i}Q^{\pi}_i\left ( s, a \right ) \\
\Leftrightarrow & r_1\left ( s, \pi\left ( s \right ) \right )+ r_2\left ( s, \pi\left ( s \right ) \right )+ \gamma P_{s,\pi\left ( s \right )}\left ( V_1^{\pi}+V_2^{\pi} \right ) \\
&\geq r_1\left ( s, a \right )+ r_2\left ( s, a \right )+ \gamma P_{s,a}\left ( V_1^{\pi}+V_2^{\pi} \right )\\
\Leftrightarrow & B_{\pi}\left ( r_1+r_2 \right )+\gamma B_{\pi}P\left ( I - \gamma G_{{\pi}} \right )^{-1}B_{{\pi}}\left ( r_1+r_2 \right ) \\
&\geq B_{a} \left ( r_1+r_2 \right )+\gamma B_{a} P\left ( I - \gamma G_{{\pi}} \right )^{-1}B_{{\pi}}\left ( r_1+r_2 \right ) \\
\Leftrightarrow &\left (B_{\pi}-B_{a}  \right )\left ( I + \gamma P\left ( I - \gamma G_{{\pi}} \right )^{-1}B_{{\pi}} \right )\left ( r_1+r_2 \right ) \geq 0\\
\Leftrightarrow &\left (B_{\pi}-B_{a}  \right )D_{\pi}\left ( r_1+r_2 \right ) \geq 0
\end{aligned}
\end{equation} 
\end{proof} 
\par
Since any solution that is consistent with \eqref{eq:uCS} ensures a unique uCS, we can borrow the idea introduced by \citeA{Lin2018} and propose a Bayesian approach. The general idea is to maximize the posterior probability of the inferred rewards, $p\left ( r_1, r_2|\pi \right )$,  which can be expressed as
\begin{equation}
p\left ( r_1, r_2|\pi \right )\propto f\left ( r_1,r_2 \right )p\left ( \pi|r_1, r_2  \right ),
\end{equation}
where $p\left ( \pi|r_1, r_2  \right )$ is the likelihood of observing $\pi$ given $r_1$ and $r_2$ and $f\left ( r_1,r_2 \right )$ is a joint prior of $r_1$ and $r_2$ that we need to specify. Recall the assumption that 
\begin{equation}
f\left ( r_1, r_2 \right )=f\left ( r_1 \right )f\left ( r_2 \right ),
\end{equation}
which allows specification of the prior over $r_1$ and $r_2$ independently. We adopt a Gaussian prior on both rewards; that is, ${r}_i\sim \mathcal{N}\left ( {\mu_{r_i}}, \Sigma_{{r_i}} \right )$, where $\mu_{r_i}$ is the mean of ${r_i}$ and $\Sigma_{{r_i}}$ is the covariance. Then the probability density function of ${r_i}$ is
\begin{equation}\label{r_density}
f\left ( {r_i} \right )= \frac{1}{\left ( 2\pi \right )^{N/2}\left | \Sigma_{{r_i}} \right |^{1/2}}\exp\left ( -\frac{1}{2}\left ( {r_i-{\mu_{r_i}}} \right )^T \Sigma^{-1}_{{r_i}} \left ( {r_i-{\mu_{r_i}}} \right ) \right ), i=1,2.
\end{equation}
\par
To model the likelihood function $p(\pi|r_1, r_2)$,  assume that the bi-policy which the two agents follow is a unique uCS given $r_1, r_2$. The likelihood is then a probability mass function given by
\begin{equation}
p\left ( \pi| r_1, r_2 \right )=\begin{cases}
1, &  \mbox{if }\pi\mbox{ is uCS for }r_1, r_2 \\
0, & \mbox{otherwise.}
\end{cases}
\end{equation} 
Thus, the optimization problem for uCS-MIRL can be formulated as 
\begin{equation}\label{uCS_model}
\begin{aligned}
\text{maximize:} \quad
&f\left ( {r_1, r_2} \right ) \\
\text{subject to:} \quad
&p\left ( {\pi}| {r_1, r_2} \right ) = 1.
\end{aligned}
\end{equation}
Equivalently, 
\begin{equation}\label{uCS_complex_program}
\begin{aligned}
\textrm{minimize:}  \quad
& \frac{1}{2}\sum_{i}\left ( {r_i-{\mu_{r_i}}} \right )^T \Sigma_{r_i}^{-1} \left ( {r_i-{\mu_{r_i}}} \right )   \\
\textrm{subject to:} \quad
&\left (B_{\pi}-B_{\pi|a}  \right )D_{\pi}\left ( r_1+r_2 \right ) \geq 0, a \in \mathcal{A}=\mathcal{A}_1\times \mathcal{A}_2.
\end{aligned}
\end{equation}
\subsection{advE-MIRL} \label{subsec:adve_mirl}
The main result characterizing the set of solutions to a two-player advE-MIRL problem is the following:
\begin{theorem} \label{thm:adve}
Given a two-player stochastic game $\left \{ \mathcal{S}, \mathcal{A}_i, r_i, P, \gamma  \right \}$, the observed bi-policy $\pi$ is an advE if and only if
\begin{equation}\label{eq:adve}
\begin{aligned}
&\left ( B_{{\pi}|a_1}-B_{{\pi}} \right )D_{{\pi}}{r}_1\leq 0,\forall a_1 \in \mathcal{A}_1 \\
&\left ( B_{{\pi}|a_2}-B_{{\pi}} \right )D_{{\pi}}{r}_2 \leq 0,\forall a_2 \in\mathcal{A}_2 \\
&\left ( B_{{\pi}|a_1}-B_{{\pi}} \right )D_{{\pi}}{r}_2 \geq 0,\forall a_1 \in\mathcal{A}_1 \\
&\left ( B_{{\pi}|a_2}-B_{{\pi}} \right )D_{{\pi}}{r}_1\geq 0,\forall a_2 \in \mathcal{A}_2,
\end{aligned}
\end{equation}
where $B_{\pi|a_{1}}$ is obtained from such a bi-policy that player 2 employs their original policy while player 1 always chooses action $a_1$ in any state (game). 
\end{theorem}
\begin{proof}\label{pf:proof_adve}
Eqs \eqref{eq:adve} contain four inequalities. In this proof, we will first show that the first and second inequalities constitute a necessary and sufficient condition for $\pi$ being a NE.  Recall that a bi-policy $\pi$ is a minimax equilibrium for a two-player zero-sum game if and only if \cite{Lin2018}
\begin{equation}
\begin{aligned}
\left [ Q^{{\pi}}\left ( s \right ) \right ]^T\pi_1\left ( s \right ) &\geq V^{{\pi}}\left ( s \right ){1}_M\\
Q^{{\pi}}\left ( s \right )\pi_2\left ( s \right )&\leq V^{{\pi}}\left ( s \right ){1}_M.
\end{aligned}
\end{equation}
Similarly, $\pi$ is a NE if and only if
\begin{equation}\label{eq:bipolicy_non0sum}
\begin{aligned}
\left [ Q_2^{{\pi}}\left ( s \right ) \right ]^T\pi_1\left ( s \right ) &\leq V_2^{{\pi}}\left ( s \right ){1}_M\\
Q_1^{{\pi}}\left ( s \right )\pi_2\left ( s \right )&\leq V_1^{{\pi}}\left ( s \right ){1}_M.
\end{aligned}
\end{equation}
Combining \eqref{V_Q_B} and \eqref{eq:bipolicy_non0sum} leads to 
\begin{equation}\label{B_Q_inequalities}
\begin{aligned}
& B_{{\pi}|a_2}\vec{Q}_2^{{\pi}}\leq B_{{\pi}}\vec{Q}_2^{{\pi}}, \forall a_2 \in \mathcal{A}_2 \\
& B_{{\pi}|a_1}\vec{Q}_1^{{\pi}}\leq B_{{\pi}}\vec{Q}_1^{{\pi}}, \forall a_1 \in \mathcal{A}_1. 
\end{aligned}
\end{equation}
Substituting \eqref{Q_complicated} into \eqref{B_Q_inequalities} and rearranging the two sides of the inequalities yields
\begin{equation} 
\begin{aligned}
&\left ( B_{{\pi}|a_1}-B_{{\pi}} \right )D_{{\pi}}{r}_1\leq 0,\forall a_1 \in \mathcal{A}_1 \\
&\left ( B_{{\pi}|a_2}-B_{{\pi}} \right )D_{{\pi}}{r}_2 \leq 0,\forall a_2 \in\mathcal{A}_2. \\
\end{aligned}
\end{equation}
%\begin{equation}\label{eq:bipolicy_0sum}
%\begin{aligned}
%\left [ Q^{{\pi}}\left ( s \right ) \right ]^T\pi_1\left ( s \right ) &\geq V^{{\pi}}\left ( s \right ){1}_M\\
%Q^{{\pi}}\left ( s \right )\pi_2\left ( s \right )&\leq V^{{\pi}}\left ( s \right ){1}_M,
%\end{aligned}
%\end{equation}
%Similarly, a NE for a 2-player general-sum game if and only if
%\begin{equation}\label{eq:bipolicy_non0sum}
%\begin{aligned}
%\left [ Q_2^{{\pi}}\left ( s \right ) \right ]^T\pi_1\left ( s \right ) &\leq V_2^{{\pi}}\left ( s \right ){1}_M\\
%Q_1^{{\pi}}\left ( s \right )\pi_2\left ( s \right )&\leq V_1^{{\pi}}\left ( s \right ){1}_M.
%\end{aligned}
%\end{equation}
%Combining \eqref{V_Q_B} and \eqref{eq:bipolicy_non0sum} leads to 
%\begin{equation}\label{B_Q_inequalities}
%\begin{aligned}
%& B_{{\pi}|a_2}Q_2^{{\pi}}\leq B_{{\pi}}Q_2^{{\pi}}, \forall a_2 \in \mathcal{A}_2 \\
%& B_{{\pi}|a_1}Q_1^{{\pi}}\leq B_{{\pi}}Q_1^{{\pi}}, \forall a_1 \in \mathcal{A}_1, 
%\end{aligned}
%\end{equation}
%where $B_{{\pi}|a_2}$ denotes that player 1 follows the bi-policy $\pi$ while player 2 always takes action $a_2$ in any state. Substituting \eqref{Q_complicated} into \eqref{B_Q_inequalities} and rearrange the two sides of the inequalities, we have 
%\begin{equation}\label{proposition_constraints_inequalities}
%\begin{aligned}
%&\left ( B_{{\pi}|a_1}-B_{{\pi}} \right )D_{{\pi}}{r}_1\leq 0,\forall a_1 \in \mathcal{A}_1 \\
%&\left ( B_{{\pi}|a_2}-B_{{\pi}} \right )D_{{\pi}}{r}_2 \leq 0,\forall a_2 \in\mathcal{A}_2, \\
%\end{aligned}
%\end{equation}
\par
We now turn to the additional feature of advE. Recall \eqref{r_adve}, it is easy to derive that a bi-policy $\pi$ for a two-player general-sum game is an advE, if and only if, in addition to \eqref{eq:bipolicy_non0sum}
\begin{equation}\label{eq:bipolicy_adve}
\begin{aligned}
\left [ Q_1^{{\pi}}\left ( s \right ) \right ]^T\pi_1\left ( s \right ) &\geq V_1^{{\pi}}\left ( s \right ){1}_M\\
Q_2^{{\pi}}\left ( s \right )\pi_2\left ( s \right )&\geq V_2^{{\pi}}\left ( s \right ){1}_M.
\end{aligned}
\end{equation}
Following similar steps to those used to derive \eqref{B_Q_inequalities}, the additional constraints \eqref{eq:bipolicy_adve} can be reduced to 
\begin{equation}\label{proposition_constraints_inequalities_adve}
\begin{aligned}
&\left ( B_{{\pi}|a_1}-B_{{\pi}} \right )D_{{\pi}}{r}_2 \geq 0,\forall a_1 \in\mathcal{A}_1 \\
&\left ( B_{{\pi}|a_2}-B_{{\pi}} \right )D_{{\pi}}{r}_1\geq 0,\forall a_2 \in \mathcal{A}_2.
\end{aligned}
\end{equation}
\end{proof} 
\par
Since it has been proved that, in a one-stage game, if an advE exists it must be unique \cite{Littman2001}, an advE for a stochastic game, must  also be unique, if it exists. Therefore, we can still use a Bayesian approach to solve advE-MIRL problems. The prior \eqref{r_density} is also valid here but the likelihood is modified as follows
\begin{equation}
p\left ( \pi| r_1, r_2 \right )=\begin{cases}
1, &  \mbox{if }\pi\mbox{ is an AdvE for }r_1, r_2 \\
0, & \mbox{otherwise.}
\end{cases}
\end{equation} 
And the optimization problem for advE-MIRL is 
\begin{equation}\label{eq:advE_complex_program}
\begin{aligned}
\textrm{minimize:}  \quad
& \frac{1}{2}\sum_{i}\left ( {r_i-{\mu_{r_i}}} \right )^T \Sigma_{r_i}^{-1} \left ( {r_i-{\mu_{r_i}}} \right )   \\
\textrm{subject to:} \quad
&\left ( B_{{\pi}|a_1}-B_{{\pi}} \right )D_{{\pi}}{r}_1\leq 0,\forall a_1 \in \mathcal{A}_1 \\
&\left ( B_{{\pi}|a_2}-B_{{\pi}} \right )D_{{\pi}}{r}_2 \leq 0,\forall a_2 \in\mathcal{A}_2 \\
&\left ( B_{{\pi}|a_1}-B_{{\pi}} \right )D_{{\pi}}{r}_2 \geq 0,\forall a_1 \in\mathcal{A}_1 \\
&\left ( B_{{\pi}|a_2}-B_{{\pi}} \right )D_{{\pi}}{r}_1\geq 0,\forall a_2 \in \mathcal{A}_2.
\end{aligned}
\end{equation}

In fact, there is a direct link between the minimax equilibrium of a competitive zero-sum game and an advE for a special zero-sum case, as the following proposition, 
\begin{proposition} \label{prop:minimax_advE}
The minimax equilibrium of a single competitive zero-sum game is an advE, and vice versa. 
\end{proposition} 
\begin{proof}
Let $r_1 = r = -r_2$. Then \eqref{eq:adve} reduces to 
\begin{equation}\label{proposition_constraints_inequalities_0sum}
\begin{aligned}
&\left ( B_{{\pi}|a_1}-B_{{\pi}} \right )D_{{\pi}}{r}\leq 0,\forall a_1 \in \mathcal{A}_1 \\
&\left ( B_{{\pi}|a_2}-B_{{\pi}} \right )D_{{\pi}}{r} \geq 0,\forall a_2 \in\mathcal{A}_2, \\
\end{aligned}
\end{equation}
which are exactly the constraints of \eqref{eq:zerosum_program}, the  necessary and sufficient conditions for $\pi$ being a minimax equilibrium for a zero-sum game (see \citeR{Lin2018}). 
\end{proof} 

From \cref{prop:minimax_advE} we can see that a advE is a more general concept for general-sum games, whereas the minimax equilibrium is specific to zero-sum games.  
\subsection{cooE-MIRL}
The main result characterizing the set of solutions to a two-player cooE-MIRL problem is the following:
\begin{theorem} \label{thm:cooe}
Given a two-player stochastic game $\left \{ \mathcal{S}, \mathcal{A}_i, r_i, P, \gamma  \right \}$, the observed bi-policy $\pi$ is an CooE if and only if
\begin{equation}\label{eq:cooe}
\begin{aligned}
&\left ( B_{{\pi}|a_1}-B_{{\pi}} \right )D_{{\pi}}{r}_1\leq 0,\forall a_1 \in \mathcal{A}_1 \\
&\left ( B_{{\pi}|a_2}-B_{{\pi}} \right )D_{{\pi}}{r}_2 \leq 0,\forall a_2 \in\mathcal{A}_2 \\
&\left ( B_{{\pi}}-B_{a}\right )D_{{\pi}}{r}_1 \geq 0,\forall a \in \mathcal{A}=\mathcal{A}_1\times \mathcal{A}_2 \\
&\left ( B_{{\pi}}-B_{a}\right )D_{{\pi}}{r}_2 \geq 0,\forall a \in \mathcal{A}=\mathcal{A}_1\times \mathcal{A}_2.
\end{aligned}
\end{equation}
\end{theorem}

In \eqref{eq:cooe}, the first two inequalities, which guarantee $\pi$ is a NE, have been established in \cref{subsec:adve_mirl}. The latter two inequalities warrant the unique property of cooE, the proof of which is outlined below.
\begin{proof}
According to the definition of cooE, $\pi$ is a cooE if and only if, for any state $s$ and pure bi-strategy $a \in \mathcal{A}=\mathcal{A}_1\times \mathcal{A}_2$,
\begin{equation}
\begin{aligned}
&\pi\left ( s \right )\in \arg \max_{a \in \mathcal{A}}Q^{\pi}_i\left ( s, a \right ) \\
\Leftrightarrow &Q^{\pi}_i\left ( s, \pi\left ( s \right ) \right )\geq Q^{\pi}_i\left ( s, a \right ) \\
\Leftrightarrow & r_i\left ( s, \pi\left ( s \right ) \right )+ \gamma P_{s,\pi\left ( s \right )} V_i^{\pi} \geq r_i\left ( s, a \right )+\gamma P_{s,a} V_i^{\pi}\\
\Leftrightarrow & B_{\pi} r_i +\gamma B_{\pi}P\left ( I - \gamma G_{{\pi}} \right )^{-1}B_{{\pi}}r_i \\
&\geq B_{a} r_i+\gamma B_{a} P\left ( I - \gamma G_{{\pi}} \right )^{-1}B_{{\pi}}r_i \\
\Leftrightarrow &\left (B_{\pi}-B_{a}  \right )\left ( I + \gamma P\left ( I - \gamma G_{{\pi}} \right )^{-1}B_{{\pi}} \right )r_i \geq 0\\
\Leftrightarrow &\left (B_{\pi}-B_{a}  \right )D_{\pi}r_i \geq 0.
\end{aligned}
\end{equation} 
\end{proof} 

Using the same reasoning as in the case of advE, it is easy to show that a cooE for a stochastic game is unique, if it exists. As a result, the Bayesian approach is also valid here, with the same prior \eqref{r_density} but a different likelihood as follows
\begin{equation}
p\left ( \pi| r_1, r_2 \right )=\begin{cases}
1, &  \mbox{if }\pi\mbox{ is an cooE for }r_1, r_2 \\
0, & \mbox{otherwise.}
\end{cases}
\end{equation} 
Hence the optimization problem for cooE-MIRL is 
\begin{equation}\label{eq:cooE_complex_program}
\begin{aligned}
\textrm{minimize:}  \quad
& \frac{1}{2}\sum_{i}\left ( {r_i-{\mu_{r_i}}} \right )^T \Sigma_{r_i}^{-1} \left ( {r_i-{\mu_{r_i}}} \right )   \\
\textrm{subject to:} \quad
&\left ( B_{{\pi}|a_1}-B_{{\pi}} \right )D_{{\pi}}{r}_1\leq 0,\forall a_1 \in \mathcal{A}_1 \\
&\left ( B_{{\pi}|a_2}-B_{{\pi}} \right )D_{{\pi}}{r}_2 \leq 0,\forall a_2 \in\mathcal{A}_2 \\
&\left ( B_{{\pi}}-B_{a}\right )D_{{\pi}}{r}_1 \geq 0,\forall a \in \mathcal{A}=\mathcal{A}_1\times \mathcal{A}_2 \\
&\left ( B_{{\pi}}-B_{a}\right )D_{{\pi}}{r}_2 \geq 0,\forall a \in \mathcal{A}=\mathcal{A}_1\times \mathcal{A}_2.
\end{aligned}
\end{equation}
\subsection{uCE-MIRL} \label{subsec:uce_mirl}
The result that characterizes the set of solutions to a two-player CE-MIRL problem is as follows:
\begin{theorem}
Given a two-player stochastic game $\left \{ \mathcal{S}, \mathcal{A}_i, r_i, P, \gamma  \right \}$, the observed bi-policy $\pi$ is a CE if and only if 
\begin{equation}\label{eq:CE} 
\vec{\pi}^T H\left ( s,a_i \right )^T\left [ H\left ( s,a_i \right )-H\left ( s,\check{a}_i \right ) \right ] D_{{\pi}}r_i \geq 0, i=1,2, \forall a_i \in \mathcal{A}_i, \check{a}_i \in \mathcal{A}_i \setminus a_i, 
\end{equation}
where $\vec{\pi}$ is restructured from $\pi$ to be a column vector of length $NM^2$. $H(s, a_i)$ is a sparse matrix of size $M\times NM^2$. Like $B_{\pi}$ defined in \eqref{r_ave}, $H(s, a_i)$ is also a linear transformation operator. Specifically, $\left [ R_1 \left ( s, a_1, . \right ) \right ]^T=H\left ( s,a_1 \right )r_1$ and $R_2 \left ( s, ., a_2 \right )=H\left ( s,a_2 \right )r_2$. Recall that here $R_1 \left ( s, a_1, . \right )$ is a $1 \times M$ row vector and $R_2 \left ( s, ., a_2 \right )$ is a $M \times 1$ column vector. Both $r_1$ and $r_2$ are $NM^2 \times 1$ column vectors.
\end{theorem}
\begin{proof}
By definition of CE, for a two-player general-sum stochastic game $\mathcal{G}$, a bi-policy $\pi$ is a CE if and only if 
\begin{equation} \label{eq:ce_sg1}
\begin{aligned}
\sum_{a_2}\pi\left ( a_1, a_2|s \right )Q_1^{\pi}\left ( s, a_1, a_2 \right ) &\geq \sum_{a_2}\pi\left ( a_1, a_2|s \right )Q_1^{\pi}\left ( s, \check{a}_1, a_2 \right ),\forall a_1 \in \mathcal{A}_1, \check{a}_1 \in \mathcal{A}_1 \setminus a_1  \\
\sum_{a_1}\pi\left ( a_1, a_2|s \right )Q_2^{\pi}\left ( s, a_1, a_2 \right ) &\geq \sum_{a_1}\pi\left ( a_1, a_2|s \right )Q_2^{\pi}\left ( s, a_1, \check{a}_2 \right ),\forall a_2 \in \mathcal{A}_2, \check{a}_2 \in \mathcal{A}_2 \setminus a_2,
\end{aligned}
\end{equation}
for all $s \in \mathcal{S}$. Rearranging \eqref{eq:ce_sg1} yields
\begin{equation} \label{eq:ce_sg2}
\begin{aligned}
\pi\left ( a_1, .|s \right )  \left ( \left [ Q_1^{\pi}\left ( s,a_1, . \right ) \right ]^T- \left [ Q_1^{\pi}\left ( s,\check{a}_1, . \right ) \right ]^T \right ) &\geq 0 \\
\left [ \pi\left ( ., a_2|s \right ) \right ]^T \left ( Q_2^{\pi}\left ( s,., a_2 \right )- Q_2^{\pi}\left ( s,., \check{a}_2 \right ) \right ) &\geq 0,
\end{aligned}
\end{equation} 
where $\pi\left ( a_1, .|s \right )$ is a row vector of $1 \times M$, spanning over all $a_2 \in \mathcal{A}_2$, and $\pi\left ( ., a_2|s \right )$ is a column vector of $M \times 1$, spanning over all $a_1 \in \mathcal{A}_1$. Recall 
\begin{equation}
Q_i^{\pi}\left ( s,a \right )=R_i \left ( s, a \right ) + \gamma \sum_{s'}p\left ( s'|s, a \right )V_i^{\pi}\left ( s' \right ).
\end{equation}
So 
\begin{equation} \label{eq:13}
\begin{aligned}
\left [ Q_1^{\pi}\left ( s,a_1, . \right ) \right ]^T&=\left [ R_1 \left ( s, a_1, .\right ) \right ]^T + \gamma p\left ( . |s, a_1, . \right )V_1^{\pi} \\
Q_2^{\pi}\left ( s,., a_2 \right ) &=R_2 \left ( s, ., a_2\right ) + \gamma p\left ( . |s, ., a_2 \right )V_2^{\pi}.
\end{aligned}
\end{equation}
Substituting \eqref{eq:13} into \eqref{eq:ce_sg2} leads to 
\begin{equation} \label{eq:ce_sg3}
\begin{aligned}
&\pi\left ( a_1, .|s \right ) \left \{ \left [ R_1 \left ( s, a_1, .\right ) \right ]^T-\left [ r_1 \left ( s, \check{a}_1, .\right ) \right ]^T + \gamma \left [ p\left ( . |s, a_1, . \right )-p\left ( . |s, \check{a}_1, . \right ) \right ]V_1^{\pi}　\right \} \geq 0 \\
&\left [ \pi\left ( ., a_2|s \right ) \right ]^T \left \{ R_2 \left ( s, ., a_2\right )-r_2 \left ( s, ., \check{a}_2\right ) + \gamma \left [ p\left ( . |s, ., a_2 \right )-p\left ( . |s, ., \check{a}_2 \right ) \right ]V_2^{\pi}　\right \} \geq 0. 
\end{aligned}
\end{equation}
\par
The above inequality can be further simplified.  It is also easy to see $p\left ( . |s, a_1, . \right )=H\left ( s,a_1 \right )P$ and $p\left ( . |s, .,a_2 \right )=H\left ( s,a_2 \right )P$. In addition, we can also have $\pi\left ( a_1, .|s \right )=\left [ H\left ( s,a_1 \right )\vec{\pi} \right ]^T=\vec{\pi}^TH\left ( s,a_1 \right )^T$, and $\pi\left ( ., a_2|s \right )=H\left ( s,a_2 \right )\vec{\pi}$. Substituting \eqref{V_comp2}
into \eqref{eq:ce_sg3} and rearranging it, we can get 
\begin{equation} \label{eq:ce_sg4}
\vec{\pi}^T H\left ( s,a_i \right )^T\left [ H\left ( s,a_i \right )-H\left ( s,\check{a}_i \right ) \right ] \left ( I + \gamma P\left ( I - \gamma G_{{\pi}} \right )^{-1}B_{{\pi}} \right )r_i \geq 0, i=1,2,
\end{equation}
Recall 
\begin{equation}
D_{{\pi}}=I + \gamma P\left ( I - \gamma G_{{\pi}} \right )^{-1}B_{{\pi}},
\end{equation}
we can express \eqref{eq:ce_sg4} compactly as 
\begin{equation} \label{eq:ce_sg5}
\vec{\pi}^T H\left ( s,a_i \right )^T\left [ H\left ( s,a_i \right )-H\left ( s,\check{a}_i \right ) \right ] D_{{\pi}}r_i \geq 0, i=1,2, \forall a_i \in \mathcal{A}_i, \check{a}_i \in \mathcal{A}_i \setminus a_i,
\end{equation} 
\end{proof} 
\par
Any sensible point that is consistent with \eqref{eq:ce_sg5} constitutes a CE for the stochastic game. Many points in the convex hull of CE, however, are less meaningful because only the uCE is of interest. Hence, we desire to find some way to choose between solutions satisfying \eqref{eq:ce_sg5}. A first idea is to maximize $\sum_{s} V^{\pi}\left ( s \right )$. Finding a uCS is a much easier problem, by contrast. This fact gives rise to another idea. Before going into details, we introduce four concepts: \emph{cooperation gap}, \emph{local uCS}, \emph{local improvement} and \emph{local reduced gap}.
\begin{definition}
The cooperation gap $I_{\text{cg}}^{\pi}\left ( s \right )$, corresponding to a starting state $s$ and a bi-policy $\pi$ in a two-player general-sum stochastic game, is the total game value difference between $\pi$ and $\pi^*$, where $\pi^*$ is any uCS; specifically,
\begin{equation*}
I_{\text{cg}}^{\pi}\left ( s \right )  =  V^{\pi^*}\left ( s \right ) - V^{\pi}\left ( s \right ), s \in \mathcal{S}.
\end{equation*}
\end{definition}

\begin{definition}
The local uCS, corresponding to a starting state $s$ and a bi-policy $\pi$ in a two-player general-sum stochastic game, is  a bi-policy for which the two players employ a uCS bi-policy $\pi^*$ at $s$ and then employ $\pi$ afterwards.
\end{definition}

\begin{definition}
The local improvement $I_{\text{imp}}^{\pi}\left ( s \right )$, corresponding to a starting state $s$ and a bi-policy $\pi$ in a two-player general-sum stochastic game, is the total game value gain by employing the local uCS.
\end{definition}

\begin{definition}
The local reduced gap $I_{\text{rg}}^{\pi}\left ( s \right )$, corresponding to a starting state $s$ and a bi-policy $\pi$ in a two-player general-sum stochastic game, is the total game value difference between a uCS and a local uCS; specifically, 
\begin{equation*}
I_{\text{rg}}^{\pi}\left ( s \right )  =  V^{\pi^*}\left ( s \right ) - Q^{\pi}\left ( s, \pi^*\left ( s \right ) \right ), s \in \mathcal{S},
\end{equation*}
and the total local improvement for $\pi$ is
\begin{equation} \label{eq:total_reduced_gap}
I_{\text{rg}}^{\pi} = \sum_{s}I_{\text{rg}}^{\pi}\left ( s \right ) = \sum_{s} V^{\pi^*}\left ( s \right ) - Q^{\pi}\left ( s, \pi^*\left ( s \right ) \right ), s \in \mathcal{S}.
\end{equation}
\end{definition}

An implication from the above definitions is that for a starting state $s$, $I_{\text{rg}}^{\pi}\left ( s \right ) = I_{\text{cg}}^{\pi}\left ( s \right ) - I_{\text{imp}}^{\pi}\left ( s \right )$, shown in \cref{fig:ce_uce_ucs}.
 
It is obvious that for a two-player general sum stochastic game, among all its CEs, the uCE is closest to its uCS in terms of the total game value, as illustrated  in \cref{fig:ce_uce_ucs}. In a uCE-MIRL problem, however, all CEs except uCE are unobservable. Therefore, we need to find a way to infer a set of rewards $\{r_1,r_2\}$ such that the observed $\pi$ is most likely the uCE of the game. 
%\begin{figure}[tbp]
%  \centering
%  \includegraphics[width=4in]{Figures/Chapter4/ce_uce.eps}\\
%  \caption{Grid games. The circle indicates A's goal and the hexagon indicates B's goal.}\label{fig:ce_uce}
%\end{figure}
%\begin{figure}[H] 
%%\captionsetup{singlelinecheck = false, justification=justified}
%  \begin{subfigure}[b]{0.50\textwidth}
%    \includegraphics[width=\textwidth]{Figures/ce_uce.eps}
%    \centering 
%    \caption{}
%  \end{subfigure}
%  \begin{subfigure}[b]{0.50\textwidth}
%    \includegraphics[width=\textwidth]{Figures/local_ucs.eps}
%    \centering  
%    \caption{}
%  \end{subfigure}
%  \caption{(A) describes the relationship between uCE, uCS and other CEs. (B) explains local uCS.}\label{fig:ce_uce_ucs}
%\end{figure}
\begin{figure}[ht]
  \centering
  \includegraphics[width=5in]{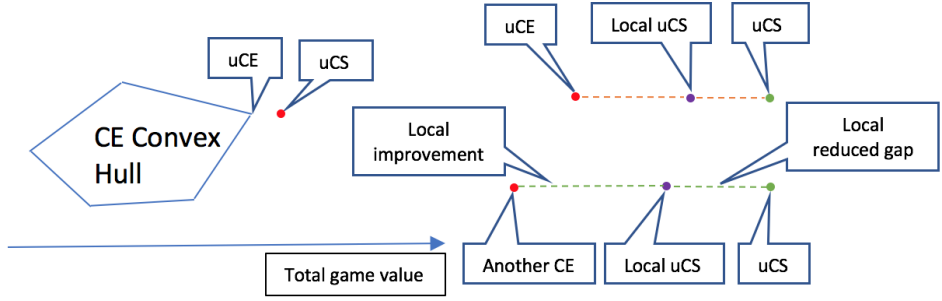}\\
  \caption{The game value distance relationship between uCE, uCS and other CEs}\label{fig:ce_uce_ucs}
\end{figure}

By definition, a local uCS improves $V^{\pi}\left ( s \right )$ by employing a uCS strategy only at current state $s$, resulting in a \emph{local improvement} with respect to $\pi$ and a local reduced gap with respect to a uCS. Adding up all those local reduced gaps over all states gives a measure of how close a bi-policy $\pi$ is to a uCS, in terms of the total game value. We now propose an important theorem that captures the relationship between the local reduced gap and uCE, as follows:
\begin{theorem} \label{thm:local_improvement}
Consider a two-person general-sum stochastic game $\Gamma$ with a collection of CEs, $\Pi_{\text{CE}}$ and a bi-policy $\pi^*$ that is a uCS. Then $\pi^*_{\text{CE}} \in \Pi_{\text{CE}}$ is a uCE if and only if its total local reduced gap $I_{\text{rg}}^{\pi^*_{\text{CE}}}$ is no greater than that of any other CE, specifically,
\begin{equation}
I_{\text{rg}}^{\pi^*_{\text{CE}}} \leq I_{\text{rg}}^{\pi_{\text{CE}}}, \forall \pi_{\text{CE}} \in \Pi_{\text{CE}}.
\end{equation}
\end{theorem}
\begin{proof}

%Recall that
%\begin{equation}
%\begin{aligned}
%V^{\pi^*}\left ( s \right ) &= V_1^{\pi^*}\left ( s \right )+V_2^{\pi^*}\left ( s \right ) = Q_1^{\pi^*}\left ( s, \pi^*\left ( s \right ) \right )+Q_2^{\pi^*}\left ( s, \pi^*\left ( s \right ) \right ) \\
%V^{\pi^*_{\text{CE}}}\left ( s \right ) &= V_1^{\pi^*_{\text{CE}}}\left ( s \right )+V_2^{\pi^*_{\text{CE}}}\left ( s \right ) = Q_1^{\pi^*_{\text{CE}}}\left ( s, \pi^*_{\text{CE}}\left ( s \right ) \right )+Q_2^{\pi^*_{\text{CE}}}\left ( s, \pi^*_{\text{CE}}\left ( s \right ) \right ) \\
%V^{\pi_{\text{CE}}}\left ( s \right ) &= V_1^{\pi_{\text{CE}}}\left ( s \right )+V_2^{\pi_{\text{CE}}}\left ( s \right ) = Q_1^{\pi_{\text{CE}}}\left ( s, \pi_{\text{CE}}\left ( s \right ) \right )+Q_2^{\pi_{\text{CE}}}\left ( s, \pi_{\text{CE}}\left ( s \right ) \right ),
%\end{aligned}
%\end{equation}
%if $\pi^*_{\text{CE}}$ is the uCE of $\Gamma$, the relationship between $\pi^*$, $\pi^*_{\text{CE}}$ and $\pi_{\text{CE}}$ is, according to our definitions of them, 
%\begin{equation} \label{eq:uCS_property}
%V^{\pi^*}\left ( s \right ) \geq V^{\pi^*_{\text{CE}}}\left ( s \right ) \geq V^{\pi_{\text{CE}}}\left ( s \right ), \forall s \in \mathcal{S}.
%\end{equation}
We first show necessity. From \eqref{eq:total_reduced_gap} and the properties of value function, we have 
\begin{equation} \label{eq:local_improvement_proof}
\begin{aligned}
&I_{\text{rg}}^{\pi^*_{\text{CE}}} - I_{\text{rg}}^{\pi_{\text{CE}}} \\
=& \left [ \sum_{s} V^{\pi^*}\left ( s \right ) - Q^{\pi^*_{\text{CE}}}\left ( s, \pi^*\left ( s \right ) \right ) \right ] - \left [ \sum_{s} V^{\pi^*}\left ( s \right ) - Q^{\pi_{\text{CE}}}\left ( s, \pi^*\left ( s \right ) \right ) \right ] \\
= & \sum_{s} Q^{\pi_{\text{CE}}}\left ( s, \pi^*\left ( s \right ) \right ) - Q^{\pi^*_{\text{CE}}}\left ( s, \pi^*\left ( s \right ) \right ) \\
= & \sum_{s} \left \{ \left [ \tilde{r}_1\left ( s, \pi^*\left ( s \right ) \right )+\tilde{r}_2\left ( s, \pi^*\left ( s \right ) \right )+ \gamma P_{s, \pi^*\left ( s \right )} V^{\pi_{\text{CE}}} \right ] \right. \\
  & \left. -\left [ \tilde{r}_1\left ( s, \pi^*\left ( s \right ) \right )+\tilde{r}_2\left ( s, \pi^*\left ( s \right ) \right )+ \gamma P_{s, \pi^*\left ( s \right )} V^{\pi^*_{\text{CE}}} \right ]  \right \}\\
= & \gamma \sum_{s} P_{s, \pi^*\left ( s \right )} \left ( V^{\pi_{\text{CE}}} - V^{\pi^*_{\text{CE}}}  \right ) .
\end{aligned}
\end{equation}
Since the definition of $\pi^*_{\text{CE}}$ implies that $V^{\pi_{\text{CE}}}\left ( s \right ) \leq V^{\pi^*_{\text{CE}}}\left ( s \right )$ for all $s$, the column vector $V^{\pi_{\text{CE}}} - V^{\pi^*_{\text{CE}}}$ is non-positive. Also, $P_{s, \pi^*\left ( s \right )}$ is a non-negative row vector as all its entries are probabilities. Therefore, $P_{s, \pi^*\left ( s \right )} \left ( V^{\pi_{\text{CE}}} - V^{\pi^*_{\text{CE}}}  \right ) \leq 0$ for all $s \in \mathcal{S}$, and consequently, $I_{\text{rg}}^{\pi^*_{\text{CE}}} \leq I_{\text{rg}}^{\pi_{\text{CE}}}$.
 
Next, we show sufficiency by assuming $I_{\text{rg}}^{\pi^*_{\text{CE}}} \leq I_{\text{rg}}^{\pi_{\text{CE}}}$, for all $\pi^*_{\text{CE}} \in \pi_{\text{CE}}$ and that $\pi^*_{\text{CE}}$ is not a uCE. Since $\pi^*_{\text{CE}}$ is not a uCE, there must exist a uCE,  $\pi_{\text{uCE}}$, such that $V^{\pi_{\text{uCE}}}\left ( s \right ) > V^{\pi^*_{\text{CE}}}\left ( s \right )$, for all $s \in \mathcal{S}$. Then from \eqref{eq:local_improvement_proof} we can conclude 
\begin{equation*}
I_{\text{rg}}^{\pi^*_{\text{CE}}} - I_{\text{rg}}^{\pi_{\text{uCE}}} = \gamma \sum_{s} P_{s, \pi^*\left ( s \right )} \left ( V^{\pi_{\text{uCE}}} - V^{\pi^*_{\text{CE}}}  \right ) > 0,
\end{equation*} 
which contradicts our assumption that $I^{\pi^*_{\text{CE}}} \leq I^{\pi_{\text{CE}}}$, for all $\pi_{\text{CE}} \in \Pi_{\text{CE}}$.
\end{proof}

The intuition behind \cref{thm:local_improvement} is: comparing to any other CE, a uCE is closer to the uCS. Its corresponding local uCS is even closer to uCS and as a result, there is less room to further shrink the local reduced gap. Hence the smaller the local reduced gap is, the more likely a CE $\pi$ is a uCE. Thus, given a CE $\pi$, a desired pair of $r_1$ and $r_2$ satisfies           
%Based on these facts, a reasonable and favourable pair of $r_1$ \& $r_2$ should have the following properties
%\begin{itemize}
%\item consistent with \eqref{eq:ce_sg5} so that the observed joint policy $pi$ is a CE;
%\item for any single game, the utilitarian game value under $\pi$ is as close to that of the uCS as possible;
%\item for any single game, the utilitarian game value under $\pi$ is relatively large because there could be many other CEs whose utilitarian game value is smaller. 
%\end{itemize} 
%The second property requires us to solve the following LP problem 
%\begin{equation*}
%\begin{aligned}
%Q^{\pi}_1\left ( s, a_1, a_2  \right )+Q^{\pi}_2\left ( s, a_1, a_2  \right ) &\leq y\left ( s \right ) \\
%V^{\pi}_1\left ( s\right )+V^{\pi}_2\left ( s\right ) &\leq y\left ( s \right )
%\end{aligned}
%\end{equation*} 
\begin{equation} \label{eq:lp_reduced_gap}
\begin{aligned}
\textrm{minimize:}  \quad
& \sum_{s} y\left ( s \right ) - V^{\pi}\left ( s\right )  \\
\textrm{subject to:} \quad
&Q^{\pi}_1\left ( s, a  \right )+Q^{\pi}_2\left ( s, a  \right ) \leq y\left ( s \right ) \\
\quad
&V^{\pi}\left ( s\right ) \leq y\left ( s \right ),
\end{aligned}
\end{equation}
for all $a \in \mathcal{A}=\mathcal{A}_1 \times \mathcal{A}_2$. 

However, $r_1 = r_2 = \mathbf{0}$ is the optimal solution to \eqref{eq:lp_reduced_gap}. The reason is that $V^{\pi}\left ( s\right )$ needs to be enlarged so that picking a uCE is achievable with higher probability. Putting all the above together, we propose the following linear programming problem to find the desired $r_1$ and $r_2$, 
\par 
\begin{equation} \label{eq:lp23}
\begin{aligned}
\textrm{maximize:}  \quad
& \sum_{s}V^{\pi}\left ( s\right ) - \lambda \left ( y\left ( s \right ) - V^{\pi}\left ( s\right ) \right ) \\
& + \mbox{regularization terms} \\
\textrm{subject to:} \quad
&Q^{\pi}_1\left ( s, a  \right )+Q^{\pi}_2\left ( s, a  \right ) \leq y\left ( s \right ) \\
\quad
&V^{\pi}\left ( s\right ) \leq y\left ( s \right )\\
&\mbox{Constraint } \eqref{eq:ce_sg5},
\end{aligned}
\end{equation}
where $\lambda$ is a regularization parameter. Expressing $V^{\pi}_i$ and $Q^{\pi}_i\left ( s, a  \right )$ as functions of $r_i$ and reformulating those inequalities more compactly in matrix notation leads to 
\begin{equation} \label{eq:uce_lp}
\begin{aligned}
\textrm{maximize:}  \quad
& \mathbf{1}_{1\times N} \times \left [ \left ( 1+\lambda \right )\left ( I - \gamma G_{\pi} \right )^{-1}B_{\pi}\left ( r_1+r_2 \right )-\lambda y \right ] \\
& + \mbox{regularization terms} \\
\textrm{subject to:} \quad
&\pi^T H\left ( s,a_i \right )^T\left [ H\left ( s,a_i \right )-H\left ( s,\check{a}_i \right ) \right ] D_{{\pi}}r_i \geq 0, \\
& i=1,2, \forall a_i \in \mathcal{A}_i, \check{a}_i \in \mathcal{A}_i \setminus a_i \\
&D_{\pi}\left ( r_1+r_2 \right ) \leq y \cdot \mathbf{1}_{NM^2}  \\
\quad
&\left ( I - \gamma G_{\pi} \right )^{-1}B_{\pi}\left ( r_1+r_2 \right ) \leq y\cdot \mathbf{1}_{N}.
\end{aligned}
\end{equation}
\par
We now discuss the regularization terms in (\ref{eq:uce_lp}). One challenging issue for MIRL is that there often exists many solutions that are equally sensible so that it is more likely than IRL to recover rewards which are far from actual ones. For example, \citeA{Lin2018} emphasize the importance of the structure of rewards. Therefore, some prior knowledge or assumption of the game, as well as the structure of the unknown rewards, is very helpful. For example, it is often assumed that, all other things being equal, an unknown reward vector is sparse \cite{Ng2000}. One easy way to incorporate this assumption is to add a penalty term to the objective function to regularize non-sparsity, which is $-\beta\left ( \left \| r_A \right \|_1 + \left \| r_B \right \|_1 \right )$, where $\beta > 0$ and $\left \| \cdot  \right \|_1$ denotes the $L_1$ norm. There might be other problem-specific knowledge/assumption regarding to reward available and taking advantage of it by incorporating it in the regularization terms will help infer higher-quality rewards.
\subsection{uNE-MIRL}
Recall that the necessary and sufficient condition for an observed bi-policy $\pi$ being a NE for a two-player general-sum stochastic game is given by 
\begin{equation}\label{eq:ne_constraint} 
\begin{aligned}
&\left ( B_{{\pi}|a_1}-B_{{\pi}} \right )D_{{\pi}}{r}_1\leq 0,\forall a_1 \in \mathcal{A}_1 \\
&\left ( B_{{\pi}|a_2}-B_{{\pi}} \right )D_{{\pi}}{r}_2 \leq 0,\forall a_2 \in\mathcal{A}_2. \\
\end{aligned}
\end{equation}
Since NE is a subset of CE, we can borrow the idea proposed in \cref{subsec:uce_mirl} and solve a uNE-MIRL problem by solving the following LP problem
\begin{equation} \label{eq:une_lp}
\begin{aligned}
\textrm{maximize:}  \quad
& \mathbf{1}_{1\times N} \times \left [ \left ( 1+\lambda \right )\left ( I - \gamma G_{\pi} \right )^{-1}B_{\pi}\left ( r_1+r_2 \right )-\lambda y \right ] \\
&+ \mbox{other problem-specific regularized terms}\\
\textrm{subject to:} \quad
&\left ( B_{{\pi}|a_1}-B_{{\pi}} \right )D_{{\pi}}{r}_1\leq 0,\forall a_1 \in \mathcal{A}_1 \\
&\left ( B_{{\pi}|a_2}-B_{{\pi}} \right )D_{{\pi}}{r}_2 \leq 0,\forall a_2 \in\mathcal{A}_2\\
&D_{\pi}\left ( r_1+r_2 \right ) \leq y \cdot \mathbf{1}_{M\times M}  \\
\quad
&\left ( I - \gamma G_{\pi} \right )^{-1}B_{\pi}\left ( r_1+r_2 \right ) \leq y.
\end{aligned}
\end{equation}

\section{Numerical Examples \Rmnum{1}: GridWorld} \label{sec:numerical_exp_1} \label{sec:experiments_general_sum}
This section describes the behaviour of our approaches (except advE-MIRL) using two grid games (GGs), shown in \cref{fig:gridworld_all}, namely GG1 for the left and GG2 for the right. These games have been used extensively in many theory-oriented MRL works \cite{Hu1998,Littman2001,Greenwald2003}. In both GGs, there are two agents, A and B, and two goals (or homes). The two agents act simultaneously and can move only one step in any of the four compass directions. When adjacent to a wall, choosing a direction into a wall results in a no-op, where the agent remains in the current position. If both agents attempt to move into the same cell, a collision occurs and they are pushed back to their original positions immediately, except for cells in the bottom row. Each agent is rewarded upon reaching its goal. However, since the reward is discounted with time, the earlier to reach the goal, the better. GG1 and GG2 are similar in basic game rules but different in board setup in two aspects. First, in GG1, the two players' goals are separate while their goals coincide in GG2. Second, in GG2, there are two barriers and if any agent attempts to move downward through the barrier from the top, then with $1/2$ probability this move fails and results in a no-op.  
\begin{figure}[ht]
  \centering
  \includegraphics[width=3.5in]{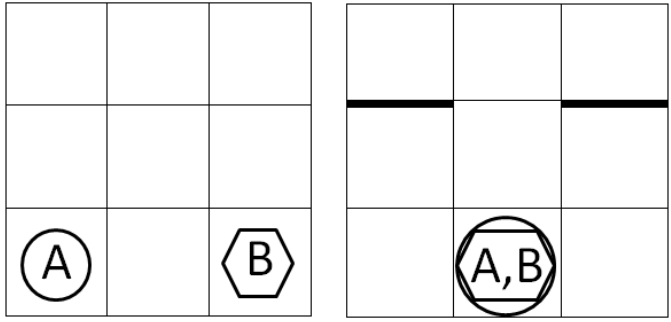}\\
  \caption{Grid games. The circle indicates A's goal and the hexagon indicates B's goal.}\label{fig:gridworld_all}
\end{figure}
We let agents A and B play the \emph{go-back-home} games together according to uCS, uCE, uNE or cooE. Our task is to recover their rewards given the equilibrium, the bi-policy, and the state transition dynamics. The \emph{basic} rewarding rule is: either player receives reward 1 (discounted with time) once reaching home and the game stops immediately, and 0 otherwise. When employing cooE, however, neither player receives reward unless they reach home \emph{simultaneously}.

Our experiments are conducted as follows. First, we apply the cooE-MRL approach proposed by \citeA{Hu2003}, and the uCE-MRL approach proposed by \citeA{Greenwald2003} to obtain cooE and uCE bi-policies, respectively. Then we develop similar Q-learning based iterative algorithms for uCS-MRL and uNE-MRL. The general procedure, namely multi-Q-learning algorithm, is the same for all these four MRL approaches and described in \cref{alg:mrl}. It is worth emphasizing that the multi-Q-learning algorithm can be applied to many variants of Q-learning problems as long as the equilibria exists and is unique \cite{Hu1998,Littman2001,Greenwald2003}. It is easy to show that uCS, uCE, uNE and cooE all meet this requirement.

The second step is to apply our uCS-MIRL, cooE-MIRL uCE-MIRL and uNE-MIRL approaches accordingly, incorporating  basic knowledge and some reasonable assumptions into  Gaussian priors for uCS and cooE. For example, one assumption is that both players' reward vectors are sparse, only depending on reaching home or not. In addition, one agent's position might have a small effect on the other agent's reward, or possibly no affect.

For each experiment, we compare recovered rewards of both players $r^{ \mbox{rec}}_A$ and $r^{ \mbox{rec}}_B$, with the true values $r_A$ and $r_B$ numerically. We use a \emph{normalized root mean squared error} (NRMSE) metric, where we first normalize a recovered reward vector $r^{ \mbox{rec}}$ on $\left [ 0, 1 \right ]$, as follows:
\begin{equation*}
r^{ \mbox{nrec}} = \frac{r^{ \mbox{rec}} - \min(r^{ \mbox{rec}})}{ \max(r^{ \mbox{rec}}) - \min(r^{ \mbox{rec}})},
\end{equation*}
and then compute
\begin{equation*} 
\mbox{NRMSE} = \frac{1}{2}\left ( \sqrt{\frac{\left \| r^{ \mbox{nrec}}_A - r_A \right \|_2}{\dim(r_A)}} + \sqrt{\frac{\left \| r^{ \mbox{nrec}}_B - r_B \right \|_2}{\dim(r_B)}}  \right )
\end{equation*} 

We compare our MIRL approach with IRL and d-MIRL approaches. First, we use an IRL approach to solve the uCS-, cooE-, uCE- and uNE-MIRL problems. Specifically, we focus on B, and try to infer its reward. Obviously, the inferred IRL reward is a function of the state and B's own action. The IRL approach we use is BIRL, proposed by \citeA{Qiao2011}. Note that the reward vector recovered from IRL can be extended to a MIRL reward vector by letting $R\left ( s, a_1, a_2 \right ) = R\left ( s, a_2\right )$ for all $a_1 \in \mathcal{A}_1$. Second, we use the d-MIRL approach to solve the above four problems. Recall that d-MIRL simply assumes agents employing a Nash equilibrium. 

To further evaluate the quality of uCS-MIRL reward, let agents $B_1$, $B_2$ and $B_3$ learn uCS-MIRL, IRL and d-MIRL rewards, respectively, and figure out their own policies $\pi_{B_1}$, $\pi_{B_2}$ and $\pi_{B_3}$. Then let these three agents play with A by using their policies while A still employs $\pi_A$ as if it plays with B. We compute their total game value over all states and compare with the true total game values, which are the maximum. For any reward, a larger the total game value indicates a better reward.  

Numerical results are summarized in \cref{tab:naed_r_general_sum} and \cref{tab:naed_v_general_sum}. \cref{tab:naed_v_general_sum} summarizes how close the total game value is to the true total game value, using the recovered reward from each method, by computing $\frac{1}{N}\sum_{s=1}^{N}\left ( V_{\text{true}}\left ( s \right )-V_{\text{recovered}}\left ( s \right ) \right )^2$, where $N$ is the number of states, $V_{\text{true}}$ and $V_{\text{recovered}}$ are true total game value vector and the one using recovered reward, respectively.
 
We can easily conclude that our MIRL approaches generate satisfactory results and performs much better than IRL and d-MIRL approaches for all the four problems.

\begin{algorithm}[tbp]
\caption{General Multi-Q-learning algorithm}\label{alg:mrl}
\begin{algorithmic}[1]
\Require $\mbox{$f$: uCS, cooE, uCE or uNE; $\alpha$: learning rate}$
\Procedure{Multi-Q}{$f,T, r_1, r_2, \alpha$}
\State \textbf{Initialize}: $s, a , Q_1, Q_2, t=0$
\While{$t < T$}
\State $\mbox{agents choose bi-strategy }a\mbox{ in state }s $
\State $\mbox{observe rewards and next state }s' $
\State \For{$i=1 \to N$}
\State $V_i\left ( s' \right )=f_i\left ( Q_1\left ( s' \right ),Q_2\left ( s' \right ) \right )$
\State $Q_i\left ( s, a \right )=\left ( 1-\alpha \right )Q_i\left ( s, a \right )+\alpha\left [ \left ( 1-\gamma \right )r_i\left ( s, a \right )+\gamma V_i\left ( s'\right )\right ]$
\State \EndFor
\State $\mbox{agents choose bi-strategy }\vec{a'} $
\State $s=s', a=a' $
\State $\mbox{decay }\alpha$
\State $t=t+1$
\EndWhile
\EndProcedure
\end{algorithmic}
\end{algorithm}
\par
%\begin{table}[t]
%  \centering
%  %\large
%  \begin{tabular}{c c c c c c c}
%    \toprule
%     & uCS-MIRL & IRL (uCS) & cooE-MIRL & uCE-MIRL & uNE-MIRL \\
%    \midrule
%    GG1 & $2.35 \times 10^{-6}$ &$1.48 \times 10^{-2}$ &$6.73 \times 10^{-4}$& $1.72\times 10^{-6}$ & $0$ \\
%    GG2 & $5.15 \times 10^{-8}$ &$1.46 \times 10^{-2}$ &$6.77 \times 10^{-4}$& $1.93\times 10^{-20}$ & $0$ \\
%    \bottomrule
%  \end{tabular}
%  \caption{NAED results for reward values comparison}
%  \label{tab:numerical_results}
%\end{table}

\begin{table}[ht]
  \centering
  %\large
  \begin{tabular}{c c c}
    \toprule
     & Grid Game \#1 & Grid Game \#2  \\
    \midrule
    uCS-MIRL & $1.50 \times 10^{-3}$ & $2.27 \times 10^{-4}$ \\
    IRL & $0.061$ & $0.060$ \\
    dMIRL & $0.104$ & $0.106$ \\
    \bottomrule
    cooE-MIRL & $0.026$ & $0.026$ \\
    IRL & $0.409$ & $0.319$ \\
    dMIRL & $0.085$ & $0.083$ \\
    \bottomrule
    uCE-MIRL & $1.30 \times 10^{-3}$ & $1.39\times 10^{-10}$ \\
    IRL & $0.287$ & $0.311$ \\
    dMIRL & $0.089$ & $0.083$ \\
    \bottomrule
    uNE-MIRL & $0$ & $0$ \\
    IRL & $0.271$ & $0.283$ \\
    dMIRL & $0.104$ & $0.103$ \\
    \bottomrule
  \end{tabular}
  \caption{NRMSE results for reward values comparison}
  \label{tab:naed_r_general_sum}
\end{table}

\begin{table}[ht]
  \centering
  %\large
  \begin{tabular}{c c c}
    \toprule
     & Grid Game \#1 & Grid Game \#2  \\
    \midrule
    uCS-MIRL & $0.099$ & $2.50 \times 10^{-4}$ \\
    IRL & $0.223$ & $0.179$ \\
    dMIRL & $0.197$ & $0.202$ \\
    \bottomrule
  \end{tabular}
  \caption{total game value comparison for uCS}
  \label{tab:naed_v_general_sum}
\end{table}

\section{Numerical Examples \Rmnum{2}: Abstract Soccer Game} \label{sec:soccer}
This section presents a demonstration of the advE-MIRL approach on a stylized soccer game. Two-player soccer games of many forms are popular among MRL researchers for algorithm demonstration and comparison purposes \cite{Greenwald2003,Lin2018,Littman1994}. \citeA{Lin2018} propose a zero-sum MIRL approach is proposed and demonstrate the performance on a game that is similar to the one used here. However, their approach requires that the game be zero-sum. In this section, we relax the zero-sum assumption and require only that the two players be foes, which enables us to rely on a weaker assumption that they employ an advE.

The soccer game (see \cref{soccer_grid_45}) is depicted as follows. Players A and B compete with each other, aiming to score by either bringing or kicking the ball (represented by a circle) into their opponents' goals (A's goal are 6 and 11, and B's goal are 10 and 15). Both players can move simultaneously either in four compass directions ending in a neighbouring cell or remain in their current cell. A ball exchange may occur with some probability in case of a collision in the same cell. A \emph{kick} action is also available to players. Each one has a perception of how likely they are to score on a shot, or the \emph{probability of a successful shot} (PSS), if kicking the ball at a given position. For simplicity, PSS is assumed to be independent of the opponent's position. The position based PSS distribution is shown in \cref{table:original_pass}.
\begin{figure}[ht]
  \centering
  \includegraphics[width=3.5in]{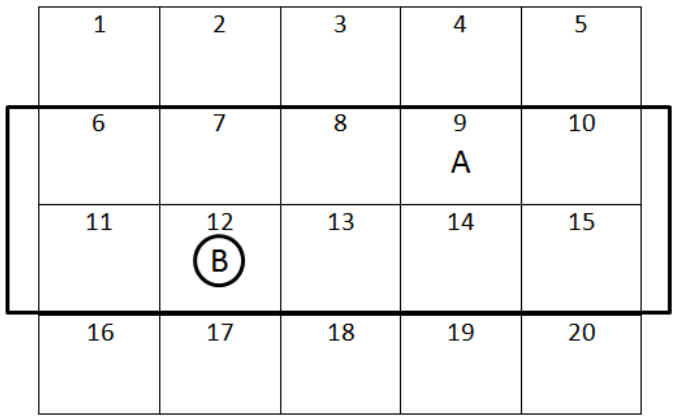}\\
  \caption{Soccer game: initial board}\label{soccer_grid_45}
\end{figure} 
\begin{table}[ht]
\centering % used for centering table
\begin{tabular}{c c c c c c}  % centered columns (4 columns)\
\hline\hline %inserts double horizontal lines
 & PSS = 0.7 & PSS = 0.5 & PSS = 0.3 & PSS = 0.1 & PSS = 0 \\  % inserts table
%heading
A & 1, 7, 12, 16 & 2, 8, 13, 17 & 3, 9, 14, 18 & 4, 10, 15, 19 & 5, 20 \\ % inserting body of the table
\hline\hline %inserts double horizontal lines
 & PSS = 0.7 & PSS = 0.5 & PSS = 0.3 & PSS = 0.1 & PSS = 0 \\  % inserts
B & 5, 9, 14, 20 & 4, 8, 13, 19 & 3, 7, 12, 18 & 2, 6, 11, 17 & 1, 16 \\
\hline %inserts single line
\end{tabular}
\caption{Original PSS distribution of each player} % title of Table
\label{table:original_pass} % is used to refer this table in the text
\end{table}

Note that a player's PSS at a particular spot is their \emph{perceived} likelihood of a scoring shot, rather than the \emph{actual} probability of a successful shot. So statistically calculating the successful shot rate from observation data would not help reflect the player's own beliefs about their shooting skills, which is the reward.      

The two players play against each other, employing a minimax equilibrium in a zero-sum game. But this information is \emph{not} available to us when solving the MIRL problem. Instead, we are given: (1) the bi-policy of the two players over all states, and; (2) the state transition dynamics (including the ball exchange rate $\beta = 0.6$). In fact, this information can be statistically calculated or estimated with sufficient observations. We simply skip this data pre-processing stage as it is not the emphasis of this paper. We then assume that the two players follow an advE and try to infer their rewards on this basis.

\subsection{Prior Specification} \label{subsec:priors}

The numerical experiments are performed using six different prior distributions.  Multivariate Gaussian distributions are used as the prior on the reward.  We test three settings for the mean of the prior distribution and two settings for the covariance.  The specifics for each parameter setting are outlined below.

\begin{itemize}
    \item \emph{Weak Mean (WM)}: The weak mean parameter setting assumes the least amount of prior knowledge about the game.  The mean parameter of the Gaussian prior is $0.5$ in each state where A has possession of the ball and $-0.5$ in each state where A does not have the ball.  The prior distribution for B is the same, i.e. the mean parameter is $0.5$ in each state where B has possession of the ball and $-0.5$ in each state where B does not have the ball.
    \item \emph{Medium Mean (MM)}: The medium mean parameter setting has more prior information about the game than the weak mean parameter setting. Specifically, if A possesses the ball and is in the row in front of the goal (1, 6, 11, and 16), the mean of the Gaussian prior is $1$ for A and $-1$ for B. Symmetrically, the mean parameter is $1$ for B and $-1$ for A if B possesses the ball and is in the row in front of the goal (5, 10, 15, and 20). In addition, whenever a player takes a shot, the mean of the prior is $0.5$ in all board positions.  Whenever the opposing player takes a shot, the mean of the prior is $-0.5$ in all board positions. The mean of the prior is $0$ for all other states and actions.
    \item \emph{Strong Mean (SM)}: The strong mean parameter setting assumes the most prior information about the game.  The strong mean parameter setting is similar to the medium mean but the parameter is 1 only when the player has has the ball and is directly in front of the goal (6 and 11 for A and 10 and 15 for B).
    \item \emph{Weak Covariance (WC)}:  The weak covariance parameter setting assumes an identity matrix for all states and actions.
    \item \emph{Strong Covariance (SC)}:  The strong covariance paramater assumes the most prior information about the game.  This covariance is built based upon the following two assumptions about action and player position:
    \begin{enumerate}
    \item when one player has the ball and takes a shot, its PSS depends only on their current position in the field. For example, let $s_1$ and $s_2$ be states where Player A is in the same position. The correlation for these states and Player A's action of shooting the ball is $\rho\left ( r_A\left ( s_1, a_A=\mbox{shoot}, a_B \right ), r_A\left ( s_2, a'_A=\mbox{shoot}, a'_B \right ) \right ) = 1$.  
    \item when one player has the ball and takes any action other than a shot (a movement or choosing to stay in the same position), their reward is state dependent, i.e. $\rho\left ( r_A\left ( s_1, a_A \ne\text{shoot}, a_B \right ), r_A\left ( s_1, a'_A\ne \text{shoot}, a'_B \right ) \right ) = 1$. 
    \end{enumerate}
\end{itemize}

\subsection{Monte Carlo Simulation using Recovered Rewards}\label{sec:monte_carlo}
This section evaluates the advE-MIRL approach. By solving an advE-MIRL problem \eqref{eq:advE_complex_program}, we recover A and B's reward vectors, over all states and all actions. There are 6 advE-MIRL reward vectors recovered corresponding to the 6 pairs of means and covariance matrices outlined in \ref{subsec:priors}. 

The zero-sum MIRL approach by \citeA{Lin2018} is validated using the abstract soccer game and measure the quality of the 6 types of zero-sum MIRL reward in two ways: numerical (Section VI, Fig.2-7) and Monte Carlo simulation, both against the true reward. They also conduct a thorough sensitivity analysis with respect to the prior (Section VII, Fig.10). Their conclusion is that though there is a trend that the more the reward numerically deviates from the true value the worse it performs in Monte Carlo simulation, measuring the deviation from the true reward is not an accurate representation of the quality of the estimated reward. Hence, simply measuring the numerical difference from the true value may lead to misleading conclusions. In addition, given the same games settings and the same prior, advE-MIRL reward is, in theory, supposed to perform no better than zero-sum MIRL due to the stronger zero-sum MIRL restriction. Taking all the above into account, we adopt a Monte Carlo simulation approach, as follows:
\begin{enumerate}
\item Create agents: 
\begin{itemize}
\item \emph{C}, which uses advE-MIRL reward; 
\item \emph{D}, which uses true reward;
\item \emph{E}, which uses zero-sum MIRL reward;
\item \emph{F}, which uses dMIRL reward;
\item \emph{G}, which uses BIRL reward.
\end{itemize}
\item Simulate competitive games:
\begin{itemize}
\item C against D;
\item C against E;
\item C against F;
\item C against G:
\end{itemize}
\end{enumerate}
Note that agents E, F and G use rewards recovered from three conventional MIRL approaches covered in \cref{sec:current_mirl}. Here we let C plays the role of A and others take the place of B (due to symmetry, two parties can switch roles as well). All those games are simulated in three different environmental settings, where the the ball exchange rates $\beta$ are $0$, $0.4$ and $1$. 5000 round games are simulated per case.  
\par
The simulation results are presented in Tables \ref{table:C_vs_D}--\ref{table:C_vs_G}, where \emph{WM}, \emph{MM}, \emph{SM}, \emph{WC} and \emph{SC} stand for \emph{weak mean}, \emph{medium mean}, \emph{strong mean}, \emph{weak covariance matrix} and \emph{strong covariance matrix}, respectively. To interpret the result, take the 2nd row of \cref{table:C_vs_E} as an example: C uses WM and SC as the prior and recovers A's advE-MIRL reward, while E uses the same prior and learns a zero-sum MIRL reward vector of B. They come up with their own minimax policies according to their learned rewards and environmental settings and play against each other. 32.28/36.40 means C beats E with probability 32.28\%, loses with probability 36.40\%, and end in a tie with probability 31.32\%, when the ball exchange rate is 1. Note that 0/0 shown in these tables indicates that both parties learned rewards such that neither is able to score a single point even if its opponent is also poorly skilled. 
\begin{table}[ht]
\centering % used for centering table
\begin{tabular}{l c c c} % centered columns (4 columns)
\hline\hline %inserts double horizontal lines
advE-MIRL Rewards &  W/L\% ($\beta = 0.4$) &  W/L\% ($\beta = 1$) &  W/L\% ($\beta = 0$)\\ [0.5ex] % inserts table
%heading
\hline % inserts single horizontal line
%True & 49.95& 50.08 & 49.43\\ % inserting body of the table
%Random & 22.78 & 19.22 & 22.91\\ % inserting body of the table
WM \& WC & 0/32.44 & 0/58.00 & 0/49.98\\ % inserting body of the table
WM \& SC & 20.40/25.46 & 20.50/38.24 & 42.88/50.16\\ % inserting body of the table
MM \& WC & 4.60/30.12 & 9.36/44.00 & 10.44/49.88\\ % inserting body of the table
MM \& SC & 24.86/24.94 & 25.10/24.80 & 49.97/50.02\\ % inserting body of the table
SM \& WC & 14.90/30.52 & 6.80/42.50 & 15.42/50.08\\ % inserting body of the table
SM \& SC & 25.26/24.80 & 25.00/24.80 & 50.14/49.86\\ % inserting body of the table
\hline %inserts single line
\end{tabular}
\caption{C vs D} % title of Table
\label{table:C_vs_D} % is used to refer this table in the text
\end{table}
\begin{table}[ht]
\centering % used for centering table
\begin{tabular}{l c c c} % centered columns (4 columns)
\hline\hline %inserts double horizontal lines
advE-MIRL Rewards &  W/L\% ($\beta = 0.4$) &  W/L\% ($\beta = 1$) &  W/L\% ($\beta = 0$)\\ [0.5ex] % inserts table
%heading
\hline % inserts single horizontal line
%True & 49.95& 50.08 & 49.43\\ % inserting body of the table
%Random & 22.78 & 19.22 & 22.91\\ % inserting body of the table
WM \& WC & 0/2.40 & 0/0 & 0/0\\ % inserting body of the table
WM \& SC & 22.76/28.94 & 32.28/36.40 & 43.14/50.14\\ % inserting body of the table
MM \& WC & 0/0 & 9.20/5.60 & 4.12/16.86\\ % inserting body of the table
MM \& SC & 24.86/25.12 & 25.04/24.96 & 49.54/50.44 \\ % inserting body of the table
SM \& WC & 11.24/10.60 & 8.80/9.18 & 16.10/24.46\\ % inserting body of the table
SM \& SC & 25.28/25.06 & 24.94/25.12 & 50.13/49.86\\ % inserting body of the table
\hline %inserts single line
\end{tabular}
\caption{C vs E} % title of Table
\label{table:C_vs_E} % is used to refer this table in the text
\end{table}
\begin{table}[ht]
\centering % used for centering table
\begin{tabular}{l c c c} % centered columns (4 columns)
\hline\hline %inserts double horizontal lines
advE-MIRL Rewards &  W/L\% ($\beta = 0.4$) &  W/L\% ($\beta = 1$) &  W/L\% ($\beta = 0$)\\ [0.5ex] % inserts table
%heading
\hline % inserts single horizontal line
%True & 49.95& 50.08 & 49.43\\ % inserting body of the table
%Random & 22.78 & 19.22 & 22.91\\ % inserting body of the table
WM \& WC & 0/0 & 0/0 & 0/0\\ % inserting body of the table
WM \& SC & 27.10/0 & 25.42/0 & 50.04/0\\ % inserting body of the table
MM \& WC & 6.04/0 & 8.64/0 & 18.02/0\\ % inserting body of the table
MM \& SC & 25.28/0 & 26.06/0 & 49.86/0 \\ % inserting body of the table
SM \& WC & 13.98/0 & 9.00/0 & 49.26/0\\ % inserting body of the table
SM \& SC & 24.90/0 & 26.08/0 & 49.90/0\\ % inserting body of the table
\hline %inserts single line
\end{tabular}
\caption{C vs F} % title of Table
\label{table:C_vs_F} % is used to refer this table in the text
\end{table}
\begin{table}[ht]
\centering % used for centering table
\begin{tabular}{l c c c} % centered columns (4 columns)
\hline\hline %inserts double horizontal lines
advE-MIRL Rewards &  W/L\% ($\beta = 0.4$) &  W/L\% ($\beta = 1$) &  W/L\% ($\beta = 0$)\\ [0.5ex] % inserts table
%heading
\hline % inserts single horizontal line
%True & 49.95& 50.08 & 49.43\\ % inserting body of the table
%Random & 22.78 & 19.22 & 22.91\\ % inserting body of the table
WM \& WC & 0/0 & 0/0 & 0/0\\ % inserting body of the table
WM \& SC & 25.10/0 & 24.84/0 & 50.12/0\\ % inserting body of the table
MM \& WC & 5.52/0 & 8.76/0 & 16.20/0\\ % inserting body of the table
MM \& SC & 28.50/10.12 & 25.12/12.00 & 49.26/20.46 \\ % inserting body of the table
SM \& WC & 14.20/0 & 8.64/0 & 44.25/0\\ % inserting body of the table
SM \& SC & 25.80/0 & 25.28/0 & 50.12/0\\ % inserting body of the table
\hline %inserts single line
\end{tabular}
\caption{C vs G} % title of Table
\label{table:C_vs_G} % is used to refer this table in the text
\end{table}

The results in Tables \ref{table:C_vs_D}--\ref{table:C_vs_G} support the following conclusions:
\begin{enumerate}
\item The advE-MIRL approach performs, if not better, comparatively with zero-sum MIRL approach given same priors. Considering that the zero-sum MIRL has a stronger constraint, our advE-MIRL approach's performance has reached its upper limit.
\item The advE-MIRL approach performs notably better than d-MIRL and BIRL approaches.
\item As for the sensitivity with respect to prior, overall, the stronger the covariance matrix or the mean is when selecting a prior, the better the solution is. In particular, the covariance matrix has a greater influence than mean in recovering a reasonable reward.
\end{enumerate}

%We need to clarify that the first rows of both tables do not imply that the win rate of B against C or D is exactly $50$\%. In fact, both rewards learned from WM \& WC and learned from d-MIRL in our simulation experiment perform so poorly that neither party is able to reach the goal even if there is no competitor at all. In other words, neither of them ever ``wins".
%\begin{figure}[htbp]
%\centering
%\includegraphics[width=4.5in]{two_metrics_comparison.eps}
%\caption{Two metrics comparison}
%\label{two_metrics_comparison}
%\end{figure}
%\begin{figure}[htbp]
%\centering
%\includegraphics[width=3.5in]{Figures/Chapter4/fancy_two_metrics_comparison_revised.eps}
%\caption{Two evaluation metrics comparison}
%\label{two_metrics_comparison}
%\end{figure}
\par

\section{Numerical Examples \Rmnum{3}: Incomplete Policy} \label{sec:incomplete_policy}
All the approaches developed in this paper rely on a strong assumption that a complete bi-policy over all states is available, which is hardly possible in practice. Therefore, it is interesting to assess how much is lost in recovered reward when using our approaches if one player's (or more) policy is not observed on all states. One way is to conduct imputation to the policy missing states. In this section, we will do a simple computational experiment in GG1 on uNE-MIRL. It is worth emphasizing that this section does not attempt to address incomplete observations issue.

Note that in GG1, there are 72 states in total. The experiment is conducted as follows. First, we randomly pick $k$ states as if the bi-strategies of those states are unavailable. Second, in the complete bi-policy, replace the bi-strategies of those $k$ states with some pre-defined bi-strategies for imputation purpose. Finally, use the imputed bi-policy as an input to \eqref{eq:ne_constraint} and recover the rewards. The imputation scheme we use for those ``unavailable" states is uniform mix-strategy - each player picks an action randomly with equal probabilities over all actions. NRMSE is used as the evaluation metric as in \cref{sec:numerical_exp_1}.

The results are summarized in \cref{tab:imputed_policy_results}. We can see that the more missing states there are, the less accurate the result is, which is in line with our expectation.
\begin{table}[ht]
\centering
\begin{tabular}{|c|c|}
\hline
\# of missing states & Result of imputed bi-policy    \\ \hline
0                    & 0             \\ \hline
1                    & 0.0521        \\ \hline
5                    & 0.0614        \\ \hline
10                   & 0.1155        \\ \hline
15                   & 0.1168        \\ \hline
\end{tabular}
\caption{NRMSE results for uNE-MIRL reward in GG1 with imputed bi-policy}
\label{tab:imputed_policy_results}
\end{table}

Addressing incomplete/partial/noisy observations is a vital topic in IRL/MIRL as it helps bridge the gap between theory and practice. Some representative work on that subject has been conducted by \citeA{Choi2011} and \citeA{Shahryari2017}. Further investigation into this topic with respect to MIRL is beyond the scope of this paper.

\section{Conclusions and Future Work} \label{sec:general_sum_conclusion}

This paper has introduced solution approaches for five general-sum MIRL problems, each distinguished by its assumption about the equilibrium being played by the observed agents. Each solution approach was demonstrated on benchmark grid-world examples. 

The advE-MIRL problem formulation requires weaker assumptions but performs comparably with zero-sum MIRL. Additionally, advE-MIRL outperforms both d-MIRL and BIRL. Both the uCS- and cooE-MIRL approaches generate good results if the estimated reward can be scaled correctly. The uCE- and uNE-MIRL approaches perform remarkably well in two benchmark grid-world examples, accurately estimating the value of the true reward. We offer three possible explanations for the high quality of the empirical results.  First, these two small GGs are well-defined in the sense that there is no chance of moving in another direction by accident once a certain direction is selected (no noise in action). Second, the bi-policy $\pi$ we use is exactly the equilibrium of interest because it is generated from a corresponding MRL-Q-learning algorithm. Third, we have incorporated strong prior information about the game, and a good solution can be achieved by tuning the regularized coefficients.     
\par
Although this paper is restricted to the two-person case, an extension of the methods to $n$-player cases would be straightforward because the equilibria we study are unique in $n$-player games, if they exist. In this sense, advE-MIRL has advantages over zero-sum MIRL as how to extend zero-sum MIRL from two-player to $n$-player is not yet clear. 

Our work is limited in at least two aspects. First, we consider only the case where both state and action spaces are discrete and limited. Second, though it is not explicitly emphasized, we use a strong assumption that a full stationary bi-policy over all states is available. In practice, it might not be possible to observe the play of the game long enough to obtain an accurate estimate of the true bi-policy over all states. Two potential directions for future work are worth pursuing. One is to derive continuous versions of the proposed approaches. The other is to treat the condition when only a partial bi-policy is available. Ideally, future methods would have the characteristic that the more complete the bi-policy observations, the more robust the recovered rewards.

Our solution approaches center on the formulation of optimization problems that can be solved in polynomial time in the size of the state and action spaces. Like MDPs, however, the sizes of stochastic games tend to scale very poorly as one moves away from toy examples, and the development of a large-scale MIRL method remains an open problem.

\vskip 0.2in
\bibliography{MIRL_database_init_ieee}
\bibliographystyle{theapa}

\end{document}